%% file: main.tex
\pdfoutput=1
\documentclass{article}

\usepackage[final]{to_arxiv} 

\usepackage[utf8]{inputenc} 
\usepackage[T1]{fontenc}    
\usepackage{hyperref}       
\usepackage{url}            
\usepackage{subfigure}
\usepackage{booktabs}       
\usepackage{amsfonts}       
\usepackage{nicefrac}       
\usepackage{microtype}      
\usepackage{amsmath}
\usepackage{amsthm}
\usepackage{amssymb}
\usepackage{amscd}
\usepackage{mathtools}
\usepackage{graphicx}
\usepackage{subfigure}
\usepackage{enumitem}
\usepackage[export]{adjustbox}
\input{macros}

\title{Robust Conditional Probabilities}
\author{
   Yoav Wald \\
  \texttt{yoav.wald@mail.huji.ac.il}
  \And
  Amir Globerson \\
  \texttt{gamir@post.tau.ac.il}
}

\begin{document}
\maketitle


\begin{abstract}
Conditional probabilities are a core concept in machine learning. For example, optimal prediction of a label $Y$ given an input $X$ corresponds to maximizing the conditional probability of $Y$ given $X$. A common approach to inference tasks is learning a model of conditional probabilities. However, these models are often based on strong assumptions (e.g., log-linear models), and hence their estimate of conditional probabilities is not robust and is highly dependent on the validity of their assumptions.

Here we propose a framework for reasoning about conditional probabilities without assuming anything about the underlying distributions, except knowledge of their second order marginals, which can be estimated from data. We show how this setting leads to guaranteed bounds on conditional probabilities, which can be calculated efficiently in a variety of settings, including structured-prediction. Finally, we apply them to semi-supervised deep learning, obtaining results competitive with variational autoencoders.
\end{abstract}


\input{intro}

\input{variational_probabilities}
\input{related}
\input{cond_probs}

\input{min_max_probs}
\input{comb_algs}
\input{experiments}

\input{discussion}
\clearpage


\bibliography{minmaxprobs,from_isf}  
\bibliographystyle{abbrv}

\newpage
\appendix
\input{appendix}

\end{document}

%% file: macros.tex
\usepackage{xifthen}
\newcommand{\ff}{\boldsymbol{f}}

\newcommand{\reals}{\mathbb{R}}

\newcommand{\prob}{\mathbb{P}}
\newcommand{\probarg}[1]{\prob\left[{#1}\right]}

\newcommand{\needcite}[1]{}

\newcommand{\be}{\begin{equation}}
\newcommand{\ee}{\end{equation}}
\newcommand{\benn}{\begin{equation*}}
\newcommand{\eenn}{\end{equation*}}
\newcommand{\bea}{\begin{eqnarray*}}
\newcommand{\eea}{\end{eqnarray*}}
\newcommand{\bean}{\begin{eqnarray}}
\newcommand{\eean}{\end{eqnarray}}

\newcommand{\pset}{\mathcal{P}}
\newcommand{\truep}{p^*}
\newcommand{\taub}{\bar{\tau}}
\newcommand{\taut}{\tilde{\tau}}

\newcommand{\uu}{\boldsymbol{u}} 
 
\newcommand{\xx}{\boldsymbol{x}} 
\newcommand{\yy}{\boldsymbol{y}} 
\newcommand{\zz}{\boldsymbol{z}}

\renewcommand{\aa}{\boldsymbol{a}}

\newcommand{\yh}{\hat{y}}

\newcommand{\muv}{{\boldsymbol{\mu}}}

\newcommand{\mub}{\bar{\mu}}

\newcommand{\deltav}{\boldsymbol{\delta}}
\newcommand{\lambdav}{\boldsymbol{\lambda}}

\newcommand{\mut}{\tilde{\mu}}
\newcommand{\muvt}{\tilde{\muv}}
\newcommand{\deltab}{\bar{\delta}}
\newcommand{\deltat}{\tilde{\delta}}
\newcommand{\ignore}[1]{}
\newcommand{\comment}[1]{}

\newcommand{\lambdab}{\bar{\lambda}}

\newcommand{\lclmargpoly}{{\cal{M}}_L}

\renewcommand{\eqref}[1]{Eq.~(\ref{#1})}

\newcommand{\figref}[1]{Figure \ref{#1}}
\newcommand{\secref}[1]{Section \ref{#1}}

\newcommand{\thmref}[1]{Thm.~\ref{#1}}
\newcommand{\lemref}[1]{Lem.~\ref{#1}}
\newcommand{\cororef}[1]{Cor.~\ref{#1}}

\newcommand{\cond}[2]{\left(#1 \mid #2\right)}

\newcommand{\expect}[2]{\mathbb{E}_{#1}\left[{#2}\right]}

\newtheorem{theorem}{Theorem}[section]
\newtheorem{lemma}[theorem]{Lemma}

\newtheorem{corollary}[theorem]{Corollary}

\newtheorem{definition}{Definition}

\newcommand{\linstatpoly}[1]{\mathcal{P}({\ifthenelse{\equal{#1}{}}{\muv}{#1}})}
\newcommand{\maxexc}[1]{\max_{p\in{\linstatpoly{}}}{\sum_{\zz\neq \yy}{p(\xx,\zz)}}}

%% file: intro.tex
\section{Introduction}

In classification tasks the goal is to predict a label $Y$ for an
object $X$. Assuming that the joint distribution of these two
variables is $p^*(\xx,\yy)$ then optimal prediction\footnote{In the
  sense of minimizing prediction error.} corresponds to returning the
label $\yy$ that maximizes the conditional probability $p^*(\yy| \xx)$.
 Thus, being able to reason about conditional probabilities is fundamental to
machine learning and probabilistic inference. 

In the fully supervised setting, one can sidestep the task of estimating 
conditional probabilities by directly learning a classifier in a discriminative fashion. However, in unsupervised or semi-supervised settings, a reliable estimate of the conditional distributions becomes important. For example, consider an unlabeled input $X$. If we had a reliable estimate of $p^*(\yy| \xx)$ we could decide whether to label the example or not, which could be used further within self-training \cite{mcclosky2006effective,weiss-EtAl:2015} or active learning contexts. Furthermore, as we show in our empirical results, such conditional probability estimates can be used as a regularizer for semi-supervised learning.
 
\comment{
 This
paper considers classification problems with categorical features and
labels, where $\mathcal{X} = [k]^n, \mathcal{Y} = [l]^m$ for some
integers $k,l,m,n$. Since learning algorithms are usually set in a
probabilistic world and assume data is generated by some distribution
$p^*(\xx,\yy)$, answering the probabilistic query $p^*\cond{\yy}{\xx}$
is of significant interest for classification.  Answers to this query
usually rely on a model $p\cond{\yy}{\xx}$ trained to approximate
$p^*\cond{\yy}{\xx}$, or a model for the joint distribution
$p(\xx,\yy)$ that approximates $p^*$ from which $p\cond{\yy}{\xx}$ is
inferred.
}


There are of course many approaches to ``modelling'' conditional distributions, from logistic regression to conditional random fields. However, these do not
come with any guarantees of approximations to the true underlying conditional distributions of $p^*$ and thus cannot be used to reliably reason about these. This is due to the fact that such models make assumptions about the conditionals (e.g., conditional independence or parametric), which are unlikely to be satisfied in practice.


As an illustrative example for our motivation and setup, consider a set of $n$ binary variables $X_1,...,X_n$ whose distribution we are
interested in. Suppose we have enough data to conclude that
$\probarg{X_1=1|X_2 =1} = 1$. This lets us reason about many other
probabilities. For example, we know that
$\probarg{X_1=1|X_2=1,\ldots,X_n=x_n}=1$ for {\em any} setting
of the $x_3,\ldots,x_n$ variables.
This is a simple but powerful observation, as it translates knowledge
about probabilities over small subsets to probability over large
subsets. Now, what happens when $\probarg{X_1=1|X_2 =1} = 0.99$? In
other words, what can we say about
$\probarg{X_1=1|X_2=1,\ldots,X_n=0}$ given information about
conditional probability $\probarg{X_i=x_i|X_j =x_j}$. As we show here, it is still
possible to reason about such conditional probabilities even under this partial knowledge.
 
Motivated by the above, we propose a novel model-free approach for reasoning about conditional probabilities. Specifically, we shall
show how conditional probabilities can be lower bounded without making strong assumptions about the underlying distribution. The only assumption we make
is that certain low-order marginals of the distribution are known. We then show how these can be used to infer lower bounds on conditional distributions
that are guaranteed to hold. One of the surprising outcomes of our analysis is that these lower bounds can be calculated efficiently, and often
have an elegant closed form. Finally, we show how these bounds can be used as a regularizer in a semi-supervised setting, obtaining results that are competitive with variational autoencoders \cite{DBLP:conf/nips/KingmaMRW14}.

%% file: variational_probabilities.tex
\section{Problem Setup}
We begin by defining notations to be used in what follows.  Let $X$ denote features and $Y$ denote labels. Assume we have $n$ features, denoted by random variables $X_1,\ldots,X_n$. If we have a single label we will denote it by $Y$. Otherwise, a multivariate label will be denoted by $Y_1,\ldots,Y_r$. We assume all variables are discrete (i.e., can take on a finite set of values). Assume that $X,Y$ are generated by some unknown underlying distribution $\truep(X,Y)$. Here we will assume that although we do not know $\truep$ we have access to the expected value of some vector function $\ff:X,Y\to \reals^d$ under $\truep$.\footnote{Abusing notation, we use $X$ to denote both the random variable and its range of values.}\footnote{For simplicity we assume the expectation is exact. Generally it is of course only approximate, but concentration bounds can be used to quantify this accuracy as a function of data size. Furthermore, most of the methods described here can be extended to inexact marginals (e.g., see \cite{dudik2007maximum} for an approach that can be applied here).} Namely we assume we are given a vector $\aa$ defined by 
$\aa = \expect{\truep}{\ff(X,Y)}$.
Since $\aa$ does not uniquely specify a distribution $\truep$, we will be interested in the set of all distributions where the expected value of $\ff(X,Y)$ is $\aa$. Denote this set by $\pset(\aa)$, namely:
\be
\pset(\aa) = \left\{ q \in \Delta :  \expect{q}{\ff(X,Y)}=\aa \right\}
\label{eq:pset}
\ee
where $\Delta$ is the probability simplex of the appropriate dimension.

We shall specifically be interested in the case where the expected values correspond to simple marginals of the distribution $\truep$, such as those of a single feature and a label:
\[
\mu_i(x_i,y) = \sum_{\bar{x}_1,\ldots,\bar{x}_n: \bar{x}_i=x_i} \truep(\bar{x}_1,\ldots,\bar{x}_n,y).
\]
Similarly we may have access to the set of pairwise marginals $\mu_{ij}(x_i,x_j,y)$ for all $i,j\in E$, where the set $E$ corresponds to edges of a graph $G$ (see also \cite{EbanICML14}). When the label is multivariate we may also incorporate marginals of the form $\mu_{lk}(y_l,y_k)$, then $(l,k)\in{E}$ and we treat labels as part of the graph.

We denote the set of all such marginals by $\muv$. And, as in \eqref{eq:pset} we define $\pset(\muv)$ to be the set of distributions whose marginals are given by $\muv$. As we shall see later, the structure of the graph $G$ will have implications on the types of bounds we can derive. Specifically, if $G$ is tree shaped (i.e., has no cycles), tight bounds can be derived.

\subsection{The Robust Conditionals Problem}
Our approach is to reason about conditional distributions using only the fact that $\truep\in\linstatpoly{}$. Our key goal is to lower bound these
conditionals, since this will allow us to conclude that certain labels are highly likely in cases where the lower bound is large. We shall also 
be interested in upper and lower bounding joint probabilities, since these will play a key role in bounding the conditionals. 

Our goal is thus to solve the following optimization problems.
\comment{
, estimated from data, to constrain a family of possible joint distributions. This is a common approach in variational inference of models \cite{DBLP:journals/ftml/WainwrightJ08} and has also been proposed for other problems \cite{DBLP:conf/nips/RoughgardenK13}. Here we bound probabilities of events in the family of these joint distributions. For notational convenience we will now consider state spaces $\mathcal{X}=[k]^n$ over $n$ discrete variables, where a subset of them $h\subseteq{[n]}$ are hidden (these will play the role of a label, possibly a structured one).
Assume we know the exact marginal distributions $\mu_{ij}(x_i,x_j)$ for pairs of variables in the set $E\subseteq{[n]\times [n]}$ and stack these values to a real vector $\muv$. The variational principles we consider are constrained by the set $\linstatpoly{}$, the set of all probability distributions that attain marginals $\muv$:
\begin{align*}
\linstatpoly{} = \{\, p\in\Delta \mid \sum_{\zz: z_i,z_j=x_i,x_j}{p(\zz)} = \mu_{ij}(x_i,x_j) \quad \forall i,j\in{E},x_i,x_j \,\}.
\end{align*}
}
\begin{align}
\min_{p\in{\linstatpoly{}}}{p(\xx,\yy)},\max_{p\in{\linstatpoly{}}}{p(\xx,\yy)}, \min_{p\in{\linstatpoly{}}}{p\cond{\yy}{\xx}}.
\label{eq:rcp}
\end{align}
In all three problems, the constraint set is linear in $p$. However, note that $p$ is specified by an exponential number of variables (one per assignment $x_1,\ldots,x_n$) and thus it is not feasible to plug these constraints into an LP solver. In terms of objective, the min and max problems are linear, and the conditional
is fractional linear. In what follows we show how all three problems can be solved efficiently for tree shaped graphs.

%% file: related.tex
\section{Related Work}
The problem of reasoning about a distribution based on its expected values has a long history, with many beautiful mathematical results. An early example is the classical Chebyshev inequality, which bounds the tail of a distribution given its first and second moments. This was significantly extended in the Chebyshev Markov Stieltjes inequality \cite{akhiezer1965classical}. More recently, various generalized Chebyshev inequalities have been developed \cite{bertsimas2005optimal,smith1995generalized,vandenberghe2007generalized}. A typical statement of these is that several moments are given, and one seeks the minimum measure of some set $S$ under any distribution that agrees with the moments. As \cite{bertsimas2005optimal} notes, most of these problems are NP hard, with isolated cases of tractability.  Such inequalities have been used to obtain minimax optimal linear classifiers in \cite{lanckriet2002robust}. The moment problems we consider here are very different from those considered previously, in terms of the finite support we require, our focus on bounding probabilities and conditional probabilities of assignments. 

The above approaches consider worst case bounds on probabilities of certain events for distributions in $\pset(\aa)$. A different approach is to pick a particular distribution in $\pset$ and use it as an approximation (or model) of $\truep$. The most common choice for such a distribution is the maximum entropy distribution in $\pset(\aa)$. Such log-linear models have found widespread use in statistics and machine learning. In particular, most graphical models can be viewed as maximum entropy distributions (e.g., see \cite{koller2009probabilistic,Lafferty01conditional}). However, the probabilities given by the maximum entropy model cannot be related to the true probabilities in any sense (e.g., upper or lower bound). This is where our approach markedly differs from entropy based assumptions.   Another approach to reducing modeling assumptions is robust optimization, where data and certain model parameters are assumed not be known precisely, and optimality is sought in a worst case adversarial setting. Such an approach has been applied to machine learning in various settings (e.g, see \cite{XuMannor09,Livni12}), establishing close links to regularization. None of these approaches considers bounding probabilities as is our focus here.

Finally, another elegant moment approach is that based on kernel mean embedding \cite{smola2007hilbert,song2013kernel}. In this approach, one maps a distribution into a set of expected values of a set of functions (possibly infinite). The key observation is that this {\em mean embedding} lies in an RKHS, and hence many operations, such as computing distribution similarity and covariances can be done implicitly. Most of the applications of this idea assume that the set of functions is rich enough to fully specify the distribution (i.e., {\em characteristic kernels} \cite{Sriperumbudur}). The focus is thus different from ours, where moments are not assumed to be fully informative, and the set $\pset(\aa)$ contains many possible distributions. It would however be interesting to study possible uses of RKHS in our setting.   

%% file: cond_probs.tex
\section{Calculating Robust Conditional Probabilities}
The optimization problems in \eqref{eq:rcp} are linear programs (LP)
and fractional LPs, where the number of variables scales exponentially
with $n$. Yet, as we show in this section and \secref{sec:closed_comb}, it turns out that in many non-trivial
cases, they can be efficiently solved. Our focus below is on the case where the pairwise marginals correspond to a set $E$ that forms a tree structured graph. The tree structure assumption is common in literature on Graphical Models, only
here we do not make an inductive assumption on the generating
distribution (i.e., we make none of the conditional independence
assumptions that are implied by tree-structured graphical models). In
the following sections we study solutions of robust conditional probabilities  under the tree assumption. We will also discuss
some extensions to the cyclic case.  Finally, note that although the derivations here are for pairwise marginals, these can be extended to the non-pairwise case by considering clique-trees \citep[e.g., see][]{wainwright2008graphical}. Pairs are used here to allow a clearer presentation. 

In what follows, we first show that the conditional lower bound has a simple structure as stated in Theorem \ref{thm:struct_cond}. This result does not immediately suggest an efficient algorithm since its denominator includes an exponentially sized LP. Next, in \secref{sec:minmax_probs} we show how this LP can be reduced to a polynomially sized one, resulting in an efficient algorithm for the lower bound. Finally, in \secref{sec:closed_comb} we show that in certain cases there is no need to use a general purpose LP solver and the problem can be solved either in closed form or via combinatorial algorithms. Detailed proofs are provided in the appendix.

\subsection{From Conditional Probabilities To Maximum Probabilities with Exclusion}
The main result of this section will reduce calculation of the robust conditional probability for $p(\yy\mid\xx)$, to one of maximizing the probability of all labels other than $\yy$. This reduction by itself will not allow for efficient calculation of the desired conditional probabilities, as the new problem is also a large LP that needs to be solved. Still the result will take us one step further towards a solution, as it reveals the probability mass a minimizing distribution $p$ will assign to $\xx,\yy$.

This part of the solution is related to a result from \cite{fromer2009lp}, where the authors derive the solution of $\min_{p\in{\linstatpoly{}}}{p(\xx,\yy)}$. They prove that under the tree assumption this problem has a simple closed form solution, given by the functional $I(\xx,y \,; \,\muv)$:
\begin{align} \label{eq:Ixmu}
I(\xx,y \,;\, \muv) = \left[ \sum_{i}{(1-d_i)\mu_i(x_i,y)} + \sum_{ij\in{E}}{\mu_{ij}(x_i,x_j,y)} \right]_+.
\end{align}
Here $ \left[ \cdot \right]_+$ denotes the ReLU function $[z]_+  = \max \{z,0\}$ and $d_i$ is the degree of node $i$ in $G$. The above expression is suitable in case of a single label, it extends naturally to the multivariate case when we consider labels as part of the graph.

It turns out that robust conditional probabilities will assign the event $\xx,\yy$ its minimal possible probability as given in \eqref{eq:Ixmu}. Moreover, it will assign all other labels their maximum possible probability. This is indeed a behaviour that may be expected from a robust bound, we formalize it in the main result for this part:
\begin{theorem} \label{thm:struct_cond}
Let $\muv$ be a vector of tree-structured pairwise marginals, then
\begin{align} \label{eq:cond_struct}
\min_{p\in{\linstatpoly{}}}p\cond{\yy}{\xx}= \frac{I(\xx,\yy \,; \muv)}{I(\xx,\yy \,; \muv) + \max_{p\in{\linstatpoly{}}}{\sum_{\bar{\yy}\neq \yy}{p(\xx,\bar{\yy})}}}.
\end{align}
\end{theorem}
The proof of this theorem is rather technical and we leave it for the appendix. 

We note that the above result also applies to the ``structured-prediction'' setting where $\yy$ is multivariate and we also assume knowledge of marginals $\mu(y_i,y_j)$. In this case, the expression for $I(\xx,\yy \,;\, \muv)$ will also include edges between $y_i$ variables, and incorporate their degrees in the graph.


The important implication of Theorem \ref{thm:struct_cond} is that it reduces the minimum conditional problem to that of probability maximization with an assignment exclusion. Namely: 
\begin{align} \label{eq:max_exclusion}
\max_{p\in{\linstatpoly{}}}{\sum_{\bar{\yy}\neq \yy}{p(\xx,\bar{\yy})}}.
\end{align}
Although this is still a problem with an exponential number of variables, we show in the next section that it can be solved efficiently.

%% file: min_max_probs.tex
\subsection{Minimizing and Maximizing Probabilities \label{sec:minmax_probs}}
To provide an efficient solution for \eqref{eq:max_exclusion}, we turn to a class of joint probability bounding problems.
Assume we constrain each variable $X_i$ and $Y_j$ to a subset $\bar{X}_i, \bar{Y}_j$ of its domain and would like to reason about the probability of this constrained set of joint assignments:
\begin{align} \label{eq:U_def}
U = \left\{ \xx,\yy \mid x_i\in{\bar{X}_i}, y_j\in{\bar{Y}_j} \quad \forall i\in{[n]}, j\in{[r]} \right\}.
\end{align}

Under this setting, an efficient algorithm for
\[\max_{p\in{\linstatpoly{}}}{\sum_{\uu\in{U\setminus (\xx,\yy)}}{p(\uu)}},\]
provides one to \eqref{eq:max_exclusion} and by the results of last section, also for robust conditional probabilities. To see this is indeed the case, assume we are given an assignment $(\xx,\yy)$. Then setting $\bar{X}_i=\{x_i\}$ for all features and $\bar{Y}_j=\{1,\ldots,|Y_j|\}$ for labels (i.e. $U$ does not restrict labels), gives exactly \eqref{eq:max_exclusion}.

To derive the algorithm, we will find a compact representation of the LP, with a polynomial number of variables and constraints. The result is obtained by using tools from the literature on Graphical Models. It shows how to formulate probability maximization problems over $U$ as problems constrained by the local marginal polytope \cite{wainwright2008graphical}. Its definition in our setting slightly deviates from its standard definition, as it does not require that probabilities sum up to $1$ \footnote{We omit the labels $Y_1,\ldots,Y_r$ from this definition for notational convenience. Formally, the consistency constraints are also enforced for edges with nodes that correspond to labels.}:
\begin{definition}
The set of locally consistent pseudo marginals over $U$ is defined as:
\begin{align*}
\lclmargpoly(U) = \{ \muvt \mid \sum_{x_i\in{\bar{X}_i}}{\mut_{ij}(x_i,x_j)} = \mut_j(x_j) \quad \forall (i,j)\in{E}, x_j\in{\bar{X}_j}   \}.
\end{align*}
The partition function of $\muvt$, $Z(\muvt)$, is given by $\sum_{x_i\in{\bar{X}_i}}{\mut_i(x_i)}$.
\end{definition}
Our observation then is that \eqref{eq:max_exclusion} can be folded into a problem with polynomially many variables and constraints, by simply maximizing the partition function over $\lclmargpoly(U)$.
\begin{theorem} \label{thm:minmaxprobs}
Let $U$ be a universe of assignments as defined in \eqref{eq:U_def}, $\xx\in{U}$ and $\muv$ a vector of tree-structured pairwise marginals, then the values of the following problems:
\begin{align*}
\max_{p\in{\linstatpoly{}}}{\sum_{\uu\in{U}}p(\uu)}, \max_{p\in{\linstatpoly{}}}{\sum_{\uu\in{U\setminus (\xx,\yy)}}{p(\uu)}},
\end{align*}
are equal (respectively) to:
\begin{align}
\max_{\muvt\in{\lclmargpoly(U)},\muvt\leq\muv}{Z(\muvt)}, \max_{\substack{\muvt\in{\lclmargpoly(U)},\muvt\leq\muv \\ I(\xx,\yy \,;\, \muvt) \leq 0}}{Z(\muvt)}.
\label{eq:max_compact_lps}
\end{align}
The LPs in \eqref{eq:max_compact_lps} involve a polynomial number of constraints and variables and can thus be solved efficiently. 
\end{theorem}
Proofs of this result can be obtained either by exploiting strong duality of LPs and the max-reparameterization property of functions that decompose over trees \cite{DBLP:journals/sac/WainwrightJW04, cowell2006probabilistic}, or by using the junction-tree theorem \cite{wainwright2008graphical}. In the appendix we provide a proof based on the latter.

\comment{
Another interesting result we found while exploring these problems is on minimisation of probabilities under this setting. The results of \cite{fromer2009lp} are naturally generalized to this scenario, using a definition of \eqref{eq:Ixmu} as a set function:
\begin{definition}
The set of all pairs of variable indices and their allowed assignments is defined as
\begin{align*}
S = \left\{ (Z,z) ~|~ z\in{\bar{Z}}, Z\in{X,Y} \right\}.
\end{align*}
Additionally we define a set functional $I(\cdot\,;\,\muv): 2^S \rightarrow \mathbb{R}$,
\begin{align} \label{eq:I_highord}
I(\tilde{S}\,;\muv) = \sum_{(i,x_i)\in{\tilde{S}}}{(1-d_i)\mu_i(x_i)} + \sum_{\substack{ ij\in{E}, \\ (i,x_i),(j,x_j)\in{\tilde{S}}}}{\mu_{ij}(x_i,x_j)}.
\end{align}
For the empty set $I(\emptyset\,;\,\muv) = 0$.
\end{definition}
An efficient way to minimise the probability of the set $U$ under $\linstatpoly{}$ is then by maximising the above functional:
\begin{theorem} \label{thm:minmaxprobs}
Let $U$ be a universe of assignments as defined in \eqref{eq:U_def}, $\xx\in{U}$ and $\muv$ a vector of tree-structured pairwise marginals. Then it holds that \[ \min_{p\in{\linstatpoly{}}}{\sum_{\uu\in{U}}p(\uu)} = \max_{\tilde{S}\in{S}}{I(\tilde{S}; \muv)}.\]
Furthermore, $I(\tilde{S}; \muv)$ is a super-modular set function.
\end{theorem}
}
To conclude this section, we restate the main result: the robust conditional probability problem \eqref{eq:rcp}  can be solved in polynomial time by combining Theorems \ref{thm:struct_cond} and \ref{thm:minmaxprobs}.
As a by-product of this derivation we also presented efficient tools for bounding answers on a large class of probabilistic queries. While this is not the focus of the current paper, these tools may be a useful in probabilistic modelling, where we often combine estimates of low order marginals with assumptions on the data generating process. Bounds like the ones presented in this section give a quantitative estimate of the uncertainty that is induced by data and circumvented by our assumptions.

%% file: comb_algs.tex
\section{Closed Form Solutions and Combinatorial Algorithms \label{sec:closed_comb}}
The results of the previous section imply that the minimum conditional can be found by solving a poly-sized LP.  Although this results in polynomial runtime, it is interesting to improve as much as possible on the complexity of this calculation.
One reason is that application of the bounds might require solving them repeatedly within some larger learning probelm. For instance, in classification tasks it may be necessary to solve \eqref{eq:cond_struct} for each sample in the dataset. An even more demanding procedure will come up in our experimental evaluation, where we learn features that result in high confidence under our bounds. There, we need to solve \eqref{eq:cond_struct} over mini-batches of training data only to calculate a gradient at each training iteration. Since using an LP solver in these scenarios is impractical, we next derive more efficient solutions to some special cases of \eqref{eq:cond_struct}.

\input{multiclass}
\input{maxflow}

%% file: multiclass.tex
\subsection{Closed Form for Multiclass Problems}
The multiclass setting is a special case of \eqref{eq:cond_struct} when $y$ is a single label variable (e.g., a digit label in mnist with values $y\in \{0,\ldots,9\}$).
In this case the problem in \eqref{eq:rcp} is:
$\min_{p\in{\linstatpoly{}}}{p\cond{y}{\xx}}$.
The solution of course depends on the type of marginals provided in $\linstatpoly{}$. Here we will assume that we have access to joint marginals of the label $y$ and pairs of features $x_i,x_j$ corresponding to edges $ij\in E$ of a graph $G$. We note that we can obtain similar results for the cases where some additional ``unlabeled'' statistics $\mu_{ij}(x_i,x_j)$ are known. 


It turns out that in both cases \eqref{eq:max_exclusion} has a simple solution. Here we write it for the case without unlabeled statistics.
\begin{lemma} \label{lem:multiclass_cond}
Let $\xx\in{\mathcal{X}}$ and $\muv$ a vector of tree-structured pairwise marginals, then
\begin{align} \label{eq:multiclass_cond}
\min_{p\in{\linstatpoly{}}}{p \cond{y}{\xx}}=\frac{I(\xx,y \,;\, \muv)}{I(\xx,y \,;\, \muv) + \sum_{\bar{y}\neq y}{\min_{ij}{\mu_{ij}(x_i,x_j,\bar{y})}}}.
\end{align}
\end{lemma}
This lemma is based on a result that states $\max_{p\in{\linstatpoly{}}}{p(\xx,\bar{y})} = \min_{ij}{\mu_{ij}(x_i,x_j,\bar{y})}$, it can either be proved by analyzing results in \thmref{thm:minmaxprobs}, or with a duality based argument which is how we prove it in the appendix.

%% file: maxflow.tex
\subsection{Combinatorial Algorithms and Connection to Maximum Flow Problems \label{sec:maxflow}}
In some cases, fast algorithms for the optimization problem in \eqref{eq:max_exclusion} can be derived by exploiting a tight connection of our problems to the Max-Flow problem. The problems are also closely related to the weighted Set Cover problem. To observe the connection to the latter, consider an instance of Set-Cover defined as follows. The universe is all assignments $\xx$. Sets are defined for each $i,j,x_i,x_j$ and are denoted by $S_{ij,x_i,x_j}$. The set $S_{ij,x_i,x_j}$ contains all assignments $\bar{x}$ whose values at $i,j$ are $x_i,x_j$. Moreoever, the set $S_{ij,x_i,x_j}$ has weight $w(S_{ij,x_i,x_j}) = \mu_{ij}(x_i,x_j)$.
Note that the number of items in sets is exponential, but there is a polynomial amount of sets. Now assume we would like to use these sets to cover some set of assignments $U$ with the minimum possible weight.
It turns out that under the tree structure assumption, this problem is closely related to the problem of maximizing probabilities.
\begin{lemma} \label{lem:setcover}
Let $U$ be a set of assignments and $\muv$ a vector of tree-structured marginals. Then:
\begin{align} \label{eq:maxprob}
\max_{p\in{\linstatpoly{}}}{\sum_{\uu\in{U}}p(\uu)},
\end{align}
has the same value as the standard LP relaxation \cite{vazirani2013approximation} of the Set-Cover problem above.
\end{lemma}
The connection to Set-Cover may not give a path to efficient algorithms, but it does illuminate some of the results presented earlier. It is simple to verify that $\min_{ij}{\mu_{ij}(x_i,x_j,\bar{y})}$ is a weight of a cover of $\xx,\bar{y}$, while \eqref{eq:Ixmu} equals one minus the weight of a set that covers all assignments but $\xx,\yy$.
A connection that we may exploit to obtain more efficient algorithms is to Max-Flow. When the graph defined by $E$ is a chain, we show in the appendix that the value of \eqref{eq:maxprob} can be found by solving a flow problem on a simple network. We note that using the same construction, \eqref{eq:max_exclusion} turns out to be Max Flow under a budget constraint \cite{DBLP:journals/networks/AhujaO95}. This may prove very beneficial for our goals, as it allows for efficient calculation of the robust conditionals we are interested in. Our conjecture is that this connection goes beyond chain graphs, but leave this for exploration in future work. The proofs for results in this section may also be found in the appendix.

%% file: experiments.tex
\section{Experiments}
To evaluate the utility of our bounds, we consider their use in settings of semi-supervised deep learning and structured prediction. For the bounds to be useful, the marginal distributions need to be sufficiently informative. In some
datasets, the raw features already provide such information, as we show in \secref{sec:struct_pred}. In other cases, such as images, a single raw feature (i.e., a pixel) does not provide sufficient information about the label. These cases 
are addressed in  \secref{sec:semisupevised} where we show how to learn new features which {\em do} result in meaningful bounds. Using deep networks to learn these features turns out to be an effective method for semi-supervised settings, reaching results close to those demonstrated by Variational Autoencoders \cite{DBLP:conf/nips/KingmaMRW14}. It would be interesting to use such feature learning methods for structured prediction too; however this requires incorporation of
the max-flow algorithm into the optimization loop, and we defer this to future work.


\subsection{Deep Semi-Supervised Learning \label{sec:semisupevised}}
Here we describe how our bounds can be used for semi-supervised learning. We learn a neural network whose last layer serve as the features $Z_i$. The marginals of these with the label $Y$ are used in our bounds (in the text we refer to these as $X_i$. Here we switch to $Z_i$ since $X_i$ are understood as the raw features of the problem. e.g., the pixel values in the image).  The features $Z_i$ will not be discrete since they are an output of a neural net. However, we will use a sigmoid activation for the last layer, so that the $Z_i$ values are bounded between $0$ and $1$. For now, let us consider the $Z_i$ as actual discrete variables with values $\{0,1\}$, and we will later explain how to overcome their non-discrete values. Given an input $\xx$, we can calculate features $\zz$, and then calculate a set of bounds for $p(y|\zz)$ for each value of $y$.  Denote this bound by $\tilde{p}_y$. Then the bound is used in two ways, depending on whether $\xx$ has a label or not. 
If $\xx$ has a label $y$, we add a standard cross-entropy term where $\tilde{p}_y$ are the logits. This pushes $\tilde{p}_y$ towards values that are maximized in the correct label. If $\xx$ is unlabeled, we want to maximize the confidence of the prediction and thus add the entropy of the distribution $q_y \propto  \tilde{p}_y$ to the objective, scaled by a regularization coefficient. This prefers solutions where $\tilde{p}_y$ is focused on one assignment. It is related to min-entropy regularization \cite{grandvalet2005semi}, but the entropy is of  a distribution induced by our bounds. Finally, for classification we use the $\arg\max$ of the distribution $\tilde{p}_y$. Namely, we do not need to learn a softmax layer as is usually done. 

The architecture used for mapping the input $\xx$ (i.e., the image) into $\zz$ is a standard multilayer perceptron (MLP), with fully connected layers, a ReLU activation at each layer, except a sigmoid in the last one. In our experiments we used hidden layers of sizes $1000,500,50$ (so $\zz$ is $50$ dimensional). We also use batch normalization and add noise in hidden layers as described in \cite{DBLP:conf/nips/RasmusBHVR15} (however we do not use any component of their unsupervised cost function).
To address the fact that $Z_i$ is not discrete, we use the natural smooth counterparts of the discrete operations. Marginals are calculated by considering $Z_i$ as an indicator variable (e.g., the probability $p(Z_i=1)$ would just be the average of the $Z_i$ values). The min probability bound is calculated as follows:
\be
\tilde{p}_y = \mathrm{softmax}_y(\frac{\bar{I}(\zz,y \,;\, \bar{\muv})}{\bar{I}(\zz,y \,;\, \bar{\muv}) + \sum_{\bar{y}\neq y}{\min_{ij}{\bar{\mu}_{ij}(z_i,z_j,\bar{y})}}})  ~,
\ee

where $\bar{I},\bar{\muv}$ are again the smoothed versions of $I,\muv$.

We compare our results with those obtained by Variational Autoencoders and Ladder Networks. Although we do not expect to get the same high accuracies these methods obtain, getting comparable numbers with a simple regularizer (compared to the elaborate techniques used in these works) like the one we suggest, is an encouraging sign for the possibility of learning features that induce high confidence. We also compare to an architecture similar to ours, but that uses minimum entropy regularization \cite{grandvalet2005semi} on a softmax layer connected to $\zz$ (i.e., it does not use our bounds at all). In this case we also add $\ell_2$ regularization on the weights of the soft-max layer, since otherwise entropy can always be driven to zero in the separable case. Finally, we also experimented with adding a hinge loss as a regularizer (as in Transductive SVM \cite{DBLP:conf/icml/Joachims99}), but omit it from the comparison because it did not yield significant improvement over a purely supervised MLP and entropy regularization.

\subsection{MNIST Dataset}
We trained the models described above on the MNIST dataset, using $100$ and $1000$ labeled samples (see \cite{DBLP:conf/nips/KingmaMRW14} for a similar setup). We set the two regularization parameters required for the entropy regularizer and the one required for our minimum probability regularizer with five fold cross validation. We used $10\%$ of the training data as a validation set and compared error rates on the $10000$ samples of the test set. Results are shown 
in Figure \ref{fig:tst_mnist}. They show that on the $1000$ sample case we are slightly outperformed by VAE and for $100$ samples we lose by $1\%$. Ladder networks outperform the other baselines.
\begin{figure} [h]
\centering
{\footnotesize
\begin{tabular}{l ccccc} 
N & Ladder \cite{DBLP:conf/nips/RasmusBHVR15}  & VAE \cite{DBLP:conf/nips/KingmaMRW14} & Robust Probs & Entropy & MLP+Noise \\
\hline
100 & $1.06 (\pm 0.37)$ & $3.33 (\pm 0.14)$ & $4.44 (\pm0.22)$ & $18.93 (\pm 0.54)$ &  $21.74(\pm 1.77)$\\
1000 & $0.84(\pm 0.08)$ & $2.40 (\pm 0.02)$ & $2.48 (\pm 0.03)$ & $3.15(\pm 0.03)$ & $5.70 (\pm 0.20 )$
\end{tabular}
}
\caption{Error rates of several semi-supervised learning methods on the MNIST dataset with few training samples.}
\label{fig:tst_mnist}
\end{figure}
\paragraph{{\bf Accuracy vs. Coverage Curves:}}
In self-training and co-training methods, a classifier adds its most confident predictions to the training set and then repeats training. A crucial factor in the success of such methods is the error in the predictions we add to the training pool. Classifiers that use confidence over unlabelled data as a regularizer are natural choices for base classifiers in such a setting. Therefore an interesting comparison to make is the accuracy we would get over the unlabeled data, had the classifier needed to choose its $k$ most confident predictions.

We plot this curve as a function of $k$ for the entropy regularizer and our min-probabilities regularizer. Samples in the unlabelled training data are sorted in descending order according to confidence. Confidence for a sample in entropy regularized MLP is calculated based on the value of the logit that the predicted label received in the output layer. For the robust probabilities classifier, the confidence of a sample is the minimum conditional probability the predicted label received.
As can be observed in \figref{fig:cvg_vs_acc}, our classifier ranks its predictions better than the entropy based method. We attribute this to our classifier being trained to give robust bounds under minimal assumptions.
\begin{figure}[h] 
\centering
\includegraphics[width=0.5\textwidth]{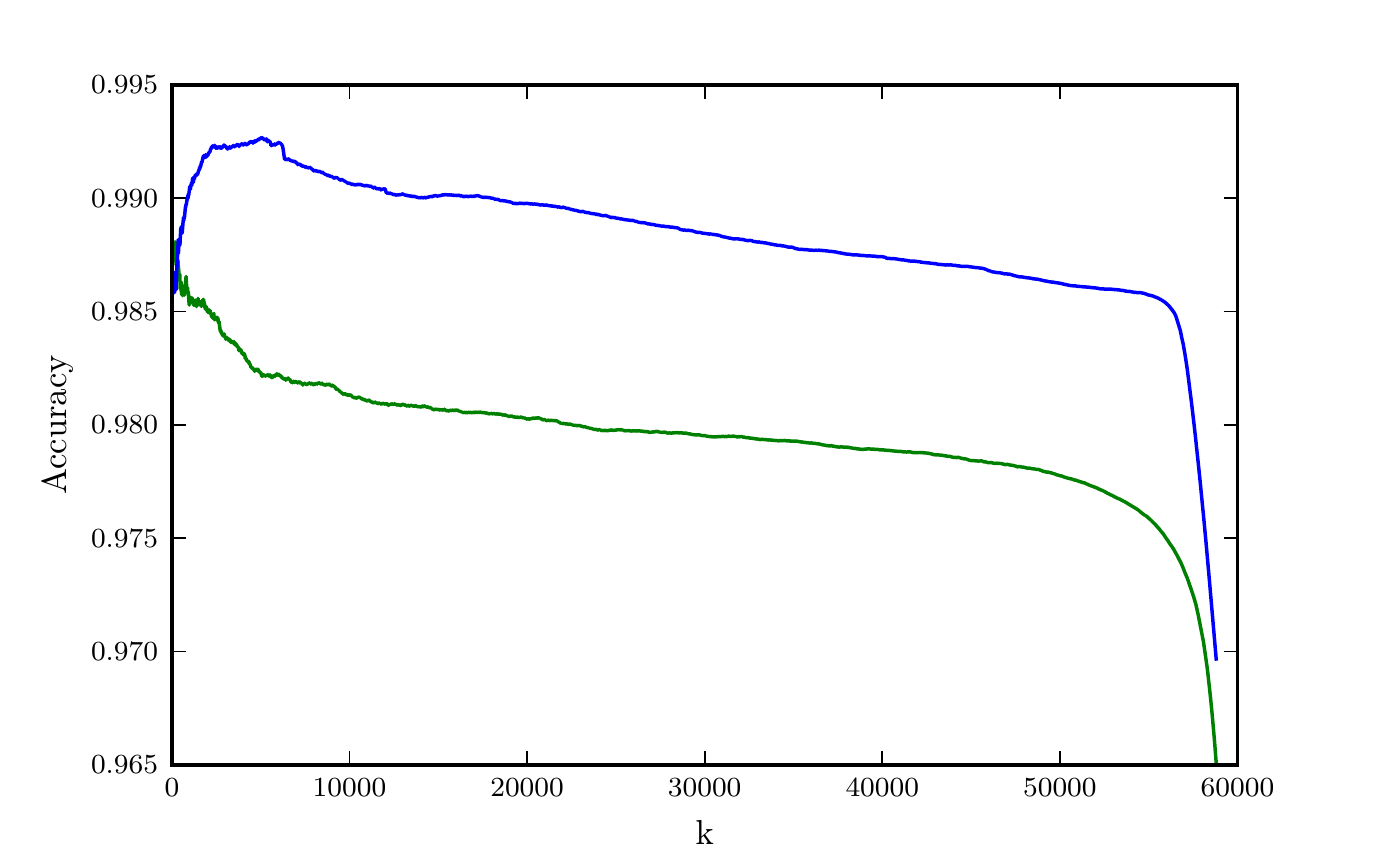} \label{fig:cvg_vs_acc}
\caption{Accuracy for $k$ most confident samples in unlabelled data. Blue curve shows results for the Robust Probabilities Classifier, green for the Entropy Regularizer. Confidence is measured by conditional probabilities and logits accordingly.}
\end{figure}

\subsection{Multilabel Structured Prediction \label{sec:struct_pred}}
As mentioned earlier, in the structured prediction setting it is more difficult to learn features that yield high certainty. We therefore provide a demonstration of our method on a dataset where the raw features are relatively informative.
The Genbase dataset taken from \cite{mulan}, is a protein classification multilabel dataset. It has $662$ instances, divided into a training set of 463 samples and a test set of 199, each sample has $1185$ binary features and 27 binary labels. We ran a structured-SVM algorithm, taken from \cite{JMLR:v15:mueller14a} to obtain a classifier that outputs a labelling $\hat{y}$ for each $\xx$ in the dataset (the error of the resulting classifier was $2\%$).
We then used our probabilistic bounds to rank the classifier's predictions by their robust conditional probabilities. The bounds were calculated based on the set of marginals 
$\mu_{ij}(x_i,y_j)$, estimated from the data for each pair of a feature and a label $X_i,Y_j$. This set of marginals corresponds to a non-tree structure and we handled it as discussed in \secref{sec:discuss}. Observing the values of our bounds, it turned out that $85\%$ of these were above $0.99$, indicating a high level of certainty that this is the correct label. Indeed only $0.59\%$ of these $85\%$ were actually errors. The remaining errors
made by the classifier were assigned min conditional probability zero by our bounds, indicating low level of certainty.

%% file: discussion.tex
\section{Discussion \label{sec:discuss}}
We presented a method for bounding conditional probabilities of a distribution based only on knowledge of its low order marginals. Our results can be viewed as a new
type of moment problem, bounding a key component of machine learning systems, namely the conditional distribution. As we show, calculating these bounds raises many challenging
optimization questions, which surprisingly result in closed form expressions in some cases.

While the results were limited to the tree structured case, some of the methods have natural extensions to the cyclic case that still result in robust estimations. For instance, the local marginal polytope in \eqref{eq:max_compact_lps} can be taken over a cyclic structure and still give a lower bound on maximum probabilities. Also in the presence of the cycles, it possible to find the spanning tree that induces the best bound on \eqref{eq:Ixmu} using a maximum spanning tree algorithm. Plugging these solutions into \eqref{eq:cond_struct} results in a tighter approximation which we used in our experiments.

Our method can be extended in many interesting directions. Here we addressed the case of discrete random variables, although we also showed in our experiments how these can be dealt with in the context of 
continuous features. It will be interesting to calculate bounds on conditional probabilities given expected values of continuous random variables. In this case, sums-of-squares characterizations play a key role \cite{lasserre2001global,parrilo2003semidefinite}, and their extension to the conditional case is an exciting challenge. It will also be interesting to study how these bounds can be used in the context of unsupervised learning. One natural approach here would be to learn constraint functions such that the lower bound is maximized.

Finally, we plan to study the implications of our approach to diverse learning settings, from self-training to active learning and safe reinforcement learning.
 

%% file: appendix.tex
{\Large\textbf{Proofs}}

This appendix provides detailed proofs of theoretical results in the paper.

We first recall a property of functions that decompose over a tree structure. Assume we have a directed tree $G$ with $n$ nodes. Denote by $r$ its root, and by $pa(i)$ the parent of node $i$. Note that
any undirected tree can be turned into a directed one by directing it away from an arbitrarily selected root. Now consider a function $\lambda(x_1,\ldots,x_n)$ over $n$ discrete variables. We will abbreviate $x_1,\ldots,x_n$ by $\xx$ wherever clear from context.
Assume that $\lambda(\xx)$ is defined as follows:
 \[ \lambda(\xx) =  \lambda_r(x_r) + \sum_{i\neq r}{\lambda_{i,pa(i)}(x_i,x_{pa(i)}) + \lambda_{i}(x_i)}. \]
 where $\lambda_r,\lambda_i$ and $\lambda_{i,j}$ are given singleton and pairwise functions. Then $\lambda(\xx)$ can be reparameterised using min ``marginals", as defined below (See \cite{cowell2006probabilistic,wainwright2008graphical} for proof of this result for max marginals and generalizations that include min operators):
	\begin{align} \label{eq:reparam}
		\lambda(\xx) &= \lambdab_r(x_r) + \sum_{i\neq r}{\lambdab_{i,pa(i)}(x_i,x_{pa(i)}) - \lambdab_{pa(i)}(x_{pa(i)})} \\
		\lambdab_i(x_i) &= \min_{\zz: z_i=x_i}{\lambda(\zz)}, \, \lambdab_{ij}(x_i,x_j) = \min_{\zz: z_i,z_j=x_i,x_j}{\lambda(\zz)} \nonumber
	\end{align}

Such $\lambda$ functions will arise, whenever we take the dual of a problem whose variables are a probability distribution constrained to satisfy some marginal distributions. Specifically, the multipliers $\lambda_i(x_i), \lambda_{ij}(x_i,x_j)$ will be those that correspond respectively to the primal constraints: \[ \sum_{\zz:z_i=x_i}{p(\zz)}=\mu_i(x_i), \sum_{\zz:z_i,z_j=x_i,x_j}{p(\zz)}=\mu_{ij}(x_i,x_j). \]

\section{Proof of \lemref{lem:multiclass_cond}}
Let us begin with the proof of \lemref{lem:multiclass_cond}, in which we derive the form of solutions used in our experiments.
\begin{proof}
	We start by writing the problem down in the following manner: \[ \min_{p\in{\linstatpoly{}}}{\frac{p(\xx,y)}{p(\xx,y) + \sum_{\yh\neq y}p(\xx,\hat{y})}}. \]
	It is obvious that in order to minimize the objective, the higher $p(\xx,\hat{y})$ is for $\hat{y}\neq y$ and the lower $p(\xx,y)$, the lower objective we get. We now notice that each of the assignments can be maximized or minimized independently, because they appear in totally distinct constraints in $\linstatpoly{}$. This is true because all constraints in $\linstatpoly{}$ are of the form:
	\[ \sum_{\zz: z_i,z_j=x_i,x_j}{p(\zz,\bar{y})} = \mu_{ij}(x_i,x_j,\bar{y}). \]
	Hence, for any pair $y_1\neq y_2$, non of the variables in $\{ p(\xx_1,y_1) \mid \xx_1\in{\mathcal{X}}\}$ appear in the same constraint with a variable in $\{ p(\xx_2,y_2) \mid \xx_2\in{\mathcal{X}}\}$, so all variables $p(\xx,\hat{y}),p(\xx,y)$ can be maximized or minimized separately. We already know from \cite{fromer2009lp} that \[ \min_{p\in{\linstatpoly{}}}{p(\xx,y)}=I(\xx,y \,;\, \muv).\]
	It is left to show that \[\max_{p\in{\linstatpoly{}}}{p(\xx,\bar{y})} = \min_{ij}{\mu_{ij}(x_i,x_j,\bar{y})},\] then the result of the lemma follows immediately. To prove the above equality we take the dual LP of the left hand side:
	\begin{align} \label{eq:dual_of_max}
	\min ~ & \lambdav \cdot \muv \\
	\text{s.t. } & \lambda(\xx,y) \geq 1 \nonumber \\
	& \lambda(\zz,\bar{y}) \geq 0 \quad \forall \zz\neq\xx \vee \bar{y}\neq y. \nonumber
	\end{align}
	Here $\lambda(\cdot)$ are the dual variables, which we can think of as a function that decomposes over a directed tree:
	\begin{align*}
	\lambda(\xx,y) =  \lambda_r(x_r,y) + \sum_{i\neq r}{\lambda_{i,pa(i)}(x_i,x_{pa(i)},y) + \lambda_{i}(x_i,y)}.
	\end{align*}
	The inner product $\lambdav\cdot \muv$ is given by:
	\begin{align}
	\sum_{i,z_i}{\lambda_i(z_i)\mu_i(z_i)} + \sum_{ij\in{E},z_i,z_j}{\lambda_{ij}(z_i,z_j)\mu_{ij}(z_i,z_j)}.
	\end{align}
	Let us take the min-reparameterization of this function and then take its expectation over a distribution $p\in{\mathcal{P}(\muv)}$. The following inequality holds for any feasible $\lambda$:
	\begin{align*}
	 \expect{p}{\lambda(\xx,y)} = &\sum_{z_r}{\mu_r(z_r)\lambdab_r(z_r,y)} + \sum_{\substack{i\neq r \\ z_i,z_{pa(i)}}}{\mu_{i,pa(i)}(z_i,z_{pa(i)})}(\lambdab(z_i,z_{pa(i)},y) - \lambdab_{pa(i)}(z_{pa(i)},y)) \\
	& \geq \mu_r(x_r)\lambdab_r(x_r,y) + \sum_{i\neq r}{\mu_{i,pa(i)}(x_i,x_{pa(i)})}(\lambdab(x_i,x_{pa(i)},y) - \lambdab_{pa(i)}(x_{pa(i)},y)).
	\end{align*}
	The inequality is true because any feasible $\lambda$ is non-negative, hence $\lambdab_r(z_r)\geq 0$ and because min-marginals over a pair of variables are always larger than those over one of them. We will conclude the proof by observing that:
	\begin{itemize}
		\item The right hand side of the inequality is a combination of the $\mu$s that are consistent with $\xx,y$ and the coefficients of this combination sum up to: \[ \lambdab_r(x_r,y) + \sum_{i\neq r}{\lambdab(x_i,x_{pa(i)},y) - \lambdab_{pa(i)}(x_{pa(i)},y)} = \lambda(\xx,y) \geq 1. \]
		The equality holds due to the reparametrization property in \eqref{eq:reparam} and $\lambda$'s feasibility. Since the sum is higher than $1$, the right hand side is also larger than any convex combination of the $\mu$s, which in turn is larger than the smallest element in the combination. We arrive at the conclusion that: \[ \expect{p}{\lambda(\xx,y)} \geq \min_{ij}{\mu_{ij}(x_i,x_j,y)}. \]
		\item It also holds that $\lambdav\cdot\muv = \expect{p}{\lambda(x)}$, hence the objective of any feasible solution is larger than $\min_{ij}{\mu_{ij}(x_i,x_j,y)}$. On the other hand, setting $\lambda_{ij}(x_i,x_j,y)=1$ for a minimizing pair $i,j$ and all other variables to $0$ results in a feasible solution with exactly this objective. It follows that this must be the optimal value of the problem.
	\end{itemize}
\end{proof}

\section{Notations for Remainder of the Proofs}\label{sec:notations}
To allow for a more convenient notation, from now on we treat labels as hidden variables. That is, instead of $n$ features and $r$ labels, we assume there are just $n$ variables $X_1,\ldots,X_n$. The first $m$ are hidden (these will play the role of a label) and the last $n-m$ are observed, where $m$ may be between $0$ and $n-1$. For an assignment $\xx$, we refer to the hidden part as $\xx_h$ and the observed as $\xx_o$. The split into hidden and observed variables will mainly serve us in the proof of \thmref{thm:struct_cond}, in other proofs it is just more convenient to not split expressions to $\xx,\yy$.

We also denote the subvector of $\muv$ over hidden variables and edges between them as $\muv_h$. That is, considering the items of $\muv$ are expressions $\mu_i(z_i),\mu_{ij}(z_i,z_j)$, $\muv_h$ is the subvector containing items where $i\in{h},i,j\in{h}$ respectively. Define a similar vector $\muv_o$ for observed variables and edges between them. The vectors $\mathbb{I}_{\xx},\mathbb{I}_{h,\xx}$ are defined to have the same indices as $\muv,\muv_h$ respectively, their value is $1$ in indices consistent with $\xx$ (i.e. $z_i,z_j = x_i,x_j$ or $z_i=x_i$ for entries that contain $\mu_{ij}(z_i,z_j),\mu_i(z_i)$ respectively) and $0$ otherwise. We will use the shorthand $\mathbf{I}_{\xx}$ for the vector $I(\xx;\mu)\mathbb{I}_{\xx}$.

Some notations related to graphical properties of hidden and observed nodes will be required. The number of connected components in the subgraph of hidden variables and edges between them is $|P_h|$, similarly for observed variables we will use $|P_o|$. The set of edges $ij$ between hidden nodes (i.e. $i,j\in{h}$) is $E_h$, between a hidden and observed node (i.e. $i\in{o},j\in{h}$ w.l.o.g) is $E_{oh}$ and between observed nodes (i.e. $i,j\in{o}$) is $E_o$. The degree of node $i$ is $d_i$ and the number of its hidden neighbors is $d_i^h$.

Finally, we define variations on the objects related to graphical models that we use in the paper. The functional $\tilde{I}(\cdot\,;\, \muv)$ is the same functional defined in \eqref{eq:Ixmu}, only without the ReLU operator:
\begin{align*}
\tilde{I}( \xx \,;\, \muv) = \sum_{i}{(1-d_i)\mu_i(x_i)} + \sum_{ij\in{E}}{\mu_{ij}(x_i,x_j)}.
\end{align*}
We will also use two variants on the local marginal polytope \cite{wainwright2008graphical}:
\begin{align*}
\lclmargpoly = \left\{ \muvt \mid \substack{\sum_{x_j}{\mut_{ij}(x_i,x_j)} = \mut_i(x_i) \quad \forall ij\in{E},x_i \\ \sum_{x_i}{\mut_{ij}(x_i,x_j)} = \mut_j(x_j) \quad \forall ij\in{E},x_j},~ \substack{\sum_{x_i}{\mut_i(x_i)} = 1 \quad\; \forall i \\ \sum_{x_i,x_j}{\mut_i(x_i,x_j)} = 1 \quad \forall i,j\in{E} } \right\}.
\end{align*}
One variant we use is $\lclmargpoly(U)$ that was defined in the paper. The other is $\lclmargpoly^{h}$, where items contain marginals only on hidden variables and edges between them:
\begin{align*}
\lclmargpoly^h = \{ \muvt \mid \substack{\sum_{x_i\in{\bar{X}_i}}{\mut_{ij}(x_i,x_j)} = \mut_j(x_j) \quad \forall (i,j)\in{E_h} \\ \sum_{x_j\in{\bar{X}_j}}{\mut_{ij}(x_i,x_j)} = \mut_i(x_i) \quad \forall (i,j)\in{E_h}}   \}.
\end{align*}

\section{Proof of \lemref{lem:setcover}}

We start by proving the connection to Set-Cover and then move on to Max-Flow.
\subsection{Connection to Set-Cover}
\begin{proof}
Let us write down the dual of \eqref{eq:maxprob}:
\begin{align} \label{eq:overall_dual}
\min ~ & \lambdav \cdot \muv \\
\text{s.t. } & \lambda_r(x_r) + \sum_{i\neq r}{\lambda_{i,pa(i)}(x_i,x_{pa(i)}) + \lambda_{i}(x_i)} \geq 0 \quad \forall \xx\notin{U} \nonumber \\
&\lambda_r(x_r) + \sum_{i\neq r}{\lambda_{i,pa(i)}(x_i,x_{pa(i)}) + \lambda_{i}(x_{i})} \geq 1 \quad \forall \xx\in{U}, \nonumber
\end{align}
This is already very similar to the LP Relaxation of Set-Cover, but with the significant difference that variables $\lambda$ are unrestricted, where in the Set-Cover LP they are non-negative. This is where the tree structure plays an important role. Consider the min-reparameterization of any feasible solution $\lambda(x)$:
\begin{align*}
\lambda(\xx) = \lambdab_r (x_r) + \sum_{i\neq r}{\lambdab_{i,pa(i)}(x_i,x_{pa(i)}) - \lambdab_{pa(i)}(x_{pa(i)})}.
\end{align*}
Since $\lambdav$ is feasible and the constraints demand that $\lambda(\xx)$ is non negative for all $\xx$, it is clear that $\lambdab_r(x_r)\geq 0$. Moreover, because $\lambdab$ is a min-reparameterization it is easy to see that $\lambdab_{i,pa(i)}(x_i,x_{pa(i)}) - \lambdab_{pa(i)}(x_{pa(i)}) \geq 0$. This is true because constraining a minimization on $x_i,x_{pa(i)}$ gives a higher result than constraining on $x_{pa(i)}$ alone.

Now let us look at the LP Relaxation of the aforementioned Set-Cover problem:
\begin{align} \label{eq:dual_tree_1}
\min ~ & \deltav\cdot\muv  \\
\text{s.t. } & \delta_r(x_r) + \sum_{i\neq r}{\delta_{i,pa(i)}(x_i,x_{pa(i)}) + \delta_{i}(x_i)} \geq 0 \quad \forall \xx\notin{U} \nonumber \\
&\delta_r(x_r) + \sum_{i\neq r}{\delta_{i,pa(i)}(x_i,x_{pa(i)}) + \delta_{i}(x_{i})} \geq 1 \quad \forall \xx\in{U}, \nonumber \\
& \deltav \geq 0 \nonumber
\end{align}
Obviously, if $\deltav$ is feasible for \eqref{eq:dual_tree_1}, setting $\lambdav = \deltav$ gives a feasible solution to \eqref{eq:overall_dual} with the same objective as $\deltav$'s in \eqref{eq:dual_tree_1}. That is, this problem is more constrained than \eqref{eq:overall_dual}. Yet given a feasible solution to \eqref{eq:overall_dual}, we can use the min-reparameterization and obtain a feasible solution to the above problem with the same objective $\lambdav\cdot\muv$:
\begin{align*}
\delta_{i}(x_i) = \begin{cases}
\lambdab_{r}(x_r) & i=r \\
0 & i\neq r
\end{cases}, \quad 
\delta_{i,pa(i)}(x_i,x_{pa(i)}) = \lambdab_{i,pa(i)}(x_i,x_{pa(i)}) - \lambdab_{pa(i)}(x_{pa(i)}).
\end{align*}
It is easy to see that because of the min-reparameteriztion property, $\delta(\xx)=\lambda(\xx)$ for all $\xx$ and $\deltav\geq 0$. This means that $\deltav$ is feasible and that the objectives are equal. To verify the latter, consider a distribution $p\in{\mathcal{P}(\muv)}$. Taking the expectations of $\deltav,\muv$ with respect to $p$ shows the equality in objectives:
\begin{align*}
\lambdav\cdot\muv = \expect{p}{\lambda(\xx)} = \expect{p}{\delta(\xx)} = \deltav\cdot\muv.
\end{align*}
We conclude that while the set cover LP Relaxation is more constrained, all feasible solutions of \eqref{eq:overall_dual} can be mapped to feasible solutions of this relaxation in a manner that preserves the objective. Hence the problems have the same value.
\end{proof}
Let us emphasize the following two points:
\begin{itemize}
	\item This part of the lemma did not exploit the specific choice of $U$ (being consisted of all assignments where variables take values in a certain set $\bar{X}_i$). That is, it holds for any choice of $U$, not only those of the form mentioned in \eqref{eq:U_def}.
	\item The constraints for $\xx\notin{U}$ in \eqref{eq:dual_tree_1} are redundant because $\deltav \geq 0$. Removing these constraints and moving back from \eqref{eq:dual_tree_1} to its dual, expressed with variables $p$, we get another formulation of \eqref{eq:maxprob}. We will use this in the next part of the proof and also later on, we thus state it as a corollary.
\end{itemize}
\begin{corollary} \label{cor:inequality_LP}
Let $U$ be a universe of assignments (not necessarily of the form in \eqref{eq:U_def}) and $\muv$ a tree-structured vector of marginals. The following LP has the same value as \eqref{eq:maxprob}:
\begin{align} \label{eq:maxprob_reform}
\max_{p\geq 0} ~ \sum_{\uu\in{U}}{p(\uu)}& && \\
\text{s.t. } \sum_{\substack{\uu\in{U} \\ u_i,u_{j}=z_i,z_{j}}}{p(\uu)} &\leq \mu_{i,j}(z_i,z_{j}) & & \forall ij\in{E},z_i,z_j \nonumber \\
\sum_{\substack{\uu\in{U} \\ u_i=z_i}}{p(\uu)} &\leq \mu_{i}(z_i) & & \forall i\in{V},z_i \nonumber
\end{align}
\end{corollary}

\subsection{Equivalence to Max-Flow}
As stated in the \secref{sec:maxflow}, when the underlying graph is a chain, \eqref{eq:maxprob} is a Max-Flow problem. The equivalence to Max-Flow is apparent when thinking of every assignment $\xx\in{U}$ as a path in a flow network. 
Assume our statistics $\muv$ are $\mu_{1,2},\mu_{2,3},\ldots,\mu_{n-1,n}$, then define a flow network with source and sink $s,t$ and  a node $(i,x_i)$ for each variable $i$ and $x_i\in{\bar{X}_i}$ (i.e. one node for each variable-assignment pair). The edges of the network are $ (i,x_i)\rightarrow(i+1,x_{i+1})$ for each $0\leq i \leq n-1$ and $x_i,x_{i+1}\in{\bar{X}_i\times \bar{X}_{i+1}}$, they will have capacity $\mu_{i,i+1}(x_i,x_{i+1})$. Additionally we will have edges $s\rightarrow(1,x_1),(n,x_n)\rightarrow t$ for each $x_1$ and $x_n$ with unbounded capacity.

It is simple to see that there is a one-to-one correspondence between paths from $s$ to $t$ and assignments in $U$. This is where $U$'s special structure, stated in \eqref{eq:U_def} of the paper comes into play. Also, the paths that go through each edge $(i,x_i)\rightarrow(i+1,x_{i+1})$ are exactly those of assignments $\zz$ where $z_i,z_{i+1}=x_i,x_{i+1}$. According to flow decomposition \cite{ford2015flows}, the LP in \eqref{eq:maxprob_reform} solves the Max-Flow problem on this network (where the flow is expressed as the sum of flows in all $s-t$ paths in the network), with a single exception that it does not contain the constraints: \[ \sum_{\substack{\uu\in{U} \\ u_i=z_i}}{p(\uu)} \leq \mu_{i}(z_i) \quad \forall i\in{V},z_i. \]
Thus to finish the proof we will get convinced that these added constraints are redundant. Consider a solution $p$ that only satisfies the constraints of pairwise marginals in \eqref{eq:maxprob_reform}, we will show it also satisfies the constraints above. Let $i\in{[n]}$ and $x_i\in{\bar{X}_i}$ and let $j$ be a neighbour of $i$ in the chain (the graph is connected, so there always is a neighbour), then: \[ \sum_{\substack{\uu\in{U} \\ u_i=x_i}}{p(\uu)} = \sum_{u_{j}\in{\bar{X}_j}}\sum_{\substack{\uu\in{U} \\ u_i,u_j=x_i,x_j}}{p(\uu)} \leq \sum_{u_j\in{\bar{X}_j}}{\mu_{ij}(x_i,u_j)} \leq \mu_i(x_i). \]
This shows the constraint is satisfied and concludes our proof.

The next proof, that of \thmref{thm:minmaxprobs}, is for results on maximizing probabilities. When the underlying graph is a chain, these results are similar to the equivalence to Max-Flow that we just proved. When the graph is not a chain, they will give an LP that does not directly correspond to a Max-Flow problem, but is still of polynomial size. That is, it can be solved efficiently with a standard LP solver, but not necessarily with a combinatorial algorithm. Our conjecture is that combinatorial algorithms can be derived for other cases, but we defer this to future work.
\section{Proof of \thmref{thm:minmaxprobs}}
The theorem reformulates the following problems:
\begin{align} \label{eq:maxprobs_primalform}
\max_{p\in{\linstatpoly{}}}{\sum_{\uu\in{U}}p(\uu)}, \max_{p\in{\linstatpoly{}}}{\sum_{\uu\in{U\setminus \xx}}{p(\uu)}}.
\end{align}
Our goal is to show that they have the same optimum as:
\begin{align} \label{eq:maxprobs_compact} \max_{\muvt\in{\lclmargpoly(U)},\muvt\leq\muv}{Z(\muvt)}, \max_{\substack{\muvt\in{\lclmargpoly(U)},\muvt\leq\muv \\ I(\xx\,;\, \muvt) \leq 0}}{Z(\muvt)}. \end{align}
\begin{proof}
To show equality of the optimal values, let us offer a mapping between feasible solutions of the pairs of problems. From our previous results, both problems in \eqref{eq:maxprobs_primalform} can be written in the form of \eqref{eq:maxprob_reform} with $U$ and $U\setminus \xx$ respectively. We will start by mapping feasible solutions of these problems to feasible solutions of \eqref{eq:maxprobs_compact}.

Choose an arbitrary root for the tree, $r\in{V}$, and turn the undirected tree to a directed one rooted in $r$. Consider a feasible solution $p$ to the reformulated problem in \eqref{eq:maxprob_reform} and define:
\begin{align*}
\tilde{\mu}_{i,pa(i)}(u_i,u_{pa(i)}) &= \sum_{\zz\in{U}: z_i,z_{pa(i)}=u_i,u_{pa(i)}}{p(\zz)} &  &\forall (u_i,u_{pa(i)}) \in{\bar{X}_i \times \bar{X}_{pa(i)}} \\
\tilde{\mu}_{i}(u_i) &= \sum_{\zz\in{U}: z_i=u_i}{p(\zz)} & &\forall u_i \in{\bar{X}_i }
\end{align*}
It is simple to prove that $\tilde{\muv}\in{\lclmargpoly(U)}$, because for any pair $ij\in{E}$ it holds that:
\[ 
\sum_{u_j\in{\bar{X}_j}}{\tilde{\mu}_{i,j}(u_i,u_j)} = \sum_{u_j\in{\bar{X}_j}}{\sum_{\zz\in{U}: z_i,z_j=u_i,u_j}{p(\zz)}} = \sum_{\zz\in{U}: z_i=u_i}{p(\zz)}=\mut_i(u_i).
\]  And from $p$'s feasibility we also get $\tilde{\muv}\leq \muv$. This can be seen from inequalities of the following type:
\[
\tilde{\mu}_{i,j}(u_i,u_j) = \sum_{\zz\in{U}: z_i,z_j=u_i,u_j}{p(\zz)} \leq \mu_{ij}(u_i,u_j).
\]
We conclude that $\muvt$ is a feasible solution to \eqref{eq:maxprobs_compact} with objective:
\begin{align*}
Z(\muvt) = \sum_{u_r\in{\bar{X}_r}}{\tilde{\mu}_r(z_r)} = \sum_{u_r\in{\bar{X}_r}}{\sum_{\zz\in{U}: z_r=u_r}{p(\zz)}} = p(U).
\end{align*}
This mapping only considered the first problem in \eqref{eq:maxprobs_primalform}. We can use the exact same construction when considering $U\setminus \xx$ as follows. Feasible solutions to \eqref{eq:maxprob_reform} are functions $p:U\setminus \xx \rightarrow \mathbb{R}_+$, so extending $p$'s domain to $U$ by setting $p(\xx)=0$, the above equations remain unaltered. It is left to show that the resulting $\muvt$ satisfies $I(\xx;\muvt)\leq 0$. If we examine the term $I(\xx;\muvt)$, when $d_i$ is the degree of node $i$ in the graph, we get that:
\begin{gather*}
\sum_{i}{(1-d_i)\mut_i(x_i)} + \sum_{ij}{\mut_{ij}(x_i,x_j)} = \sum_{\uu\in{U}}{\alpha_{\uu} p(\uu)}, \\
\alpha_{\uu} \triangleq \sum_{i}{\mathbb{I}_{u_i=x_i}} - \sum_{ij}{\mathbb{I}_{(u_i=x_i) \vee (u_j=x_j)}}.
\end{gather*}
Simple counting arguments show that $\alpha_{\xx}=1$, while $\alpha_{\uu}\leq 0$ for all $\uu\neq\xx$. Since we set $p(\xx)=0$, it follows that $\sum_{\uu\in{U}}{\alpha_{\uu} p(\uu)} \leq 0$ and also $I(\xx;\muvt)$.

It is left to provide a mapping from solutions of \eqref{eq:maxprobs_compact} to solutions of \eqref{eq:maxprobs_primalform}.
We will provide a proof for the case where
\[
U = \left\{ \uu \mid u_i\in{\bar{X}_i} \quad \forall i\in{[n]} \right\}.
\]
More specifically, we will construct a function $p : U \rightarrow \mathbb {R}_+$ whose marginals are $\muvt$ and summing it over all of its domain gives $Z(\muvt)$. The construction is the same one used when proving that the local marginal polytope is equal to the marginal polytope for tree graphs \cite{wainwright2008graphical}. To complete the proof, we will also need to show a construction when $p$'s domain is $U\setminus \xx$ (and $U$ defined the same as above). We refer the reader to \cite{fromer2009lp} where this detailed construction can be found. There the sum of $p$ over its domain is $1$, yet applying this construction to $\muvt$ gives a function that sums up to $Z(\muvt)$.

The function $p$ we suggest for the problem over domain $U$ is:
\begin{align*}
p(\uu) = \tilde{\mu}_r(u_r)\prod_{i\neq r}{\frac{\tilde{\mu}_{i,pa(i)}(u_i,u_{pa(i)})}{\tilde{\mu}_{pa(i)}(u_i)} }.
\end{align*}
Assume $r$ is set arbitrarily and $1,\ldots,n$ is a topological ordering of the nodes. Notice that any choice of $r$ and an ordering yields the same function $p$. It is simple to see that the function marginalizes to $\mut$ if we let $ij\in{E}$, set $i$ as the root and eliminate all variables other than $i,j$. To show that $p$'s sum over its domain $U$ is exactly the partition function, eliminate all the variables to get:
\begin{align*}
&\sum_{\xx\in{U}}p(\xx) = \\ &\sum_{u_1\in{\bar{X}_1}}{\tilde{\mu}_1(u_1)} \left( \sum_{u_2\in{\bar{X}_2}}{\frac{\tilde{\mu}_{2,pa(2)}(u_2,u_{pa(2)})}{\tilde{\mu}_{pa(2)}(u_2)}} \ldots \left( \sum_{u_n\in{\bar{X}_n}}{\frac{\tilde{\mu}_{n,pa(n)}(u_{n},u_{pa(n)})}{\tilde{\mu}_{pa(n)}(u_{pa(n)})}} \right) \right) = \sum_{u_1\in{\bar{X}_1}}{\tilde{\mu}_1(u_1)}.
\end{align*}
Here we implicitly numbered the root node as $1$. To conclude, we showed a mapping from $\mut$ to a function $p$ that is feasible for \eqref{eq:maxprobs_primalform}, completing the proof.

For the case $U\setminus \xx$, as stated earlier, \cite{fromer2009lp} offer a construction of a function that marginalizes to $\tilde{\muv}$ and achieves $p(\xx) = I(\xx\,;\,\muv)$. Thus enforcing $I(\xx\,;\, \muv)\leq 0$ ensures there is a mapping from $\tilde{\muv}$ to a function $p$ with the same objective.

Notice the equality in the above equation holds because of $U$'s special structure that includes \textbf{all} the assignments that take values in sets $\bar{X}_i$. Different choices of $U$ do not necessarily yield this equation, thus the theorem does not hold for all choices of $U$.\end{proof}

\section{Proof of \thmref{thm:struct_cond}}
We recall the problem at hand of minimizing conditional probabilities:
\begin{align*}
\min_{p\in{\mathcal{P}(\muv)}}{p(\xx_h ~|~ \xx_o)},
\end{align*}
where we assume w.l.o.g that $\xx_h = x_1,\ldots,x_m$ are hidden variables, $\xx_o = x_{m+1},\ldots,x_n$ are observed, and $\xx$ is the fixed assignment to both.
Using the Charnes-Cooper variable transformation \cite{charnes1962programming} between $p(\zz_h,\zz_o)$ and $\frac{p(\zz_h,\zz_o)}{p(\xx_o)}$ for all $\zz$, and taking the dual of the resulting LP, we arrive at the following problem:
\begin{align} \label{eq:cond1}
\max ~ & \lambda_{\xx} \\
\text{s.t. } & \lambda_r(z_r) + \sum_{i\neq r}{\lambda_{i,pa(i)}(z_i,z_{pa(i)}) + \lambda_{i}(z_i)} \leq 0 & & \forall \zz: ~ \zz_o\neq \xx_o \nonumber \\
&\lambda_r(z_r) + \sum_{i\neq r}{\lambda_{i,pa(i)}(z_i,z_{pa(i)}) + \lambda_{i}(z_{i})} \leq -\lambda_{\xx} & & \forall \zz: ~ \zz_o = \xx_o, \zz_h \neq \xx_h \nonumber, \\
&\lambda_r(x_r) + \sum_{i\neq r}{\lambda_{i,pa(i)}(x_i,x_{pa(i)}) + \lambda_{i}(x_{i})} \leq 1-\lambda_{\xx} & & \nonumber \\
&\lambda \cdot \muv \geq 0 \nonumber.
\end{align}
The transformation is correct under the assumption that $p(\xx_o)>0$, which is reasonable to assume when we observe $\xx_o$ and try to infer $\xx_h$.

The rest of the proof can now be decomposed into two main parts, one manipulates \eqref{eq:cond1} and the other manipulates the second problem in \eqref{eq:maxprobs_compact}:
\begin{lemma} \label{lem:cond_closed_form_helper}
Let $U$ be a set of the shape defined in \eqref{eq:U_def} of the paper and $\muv$ a vector of tree shaped marginals.
If 
\begin{align} \label{eq:full_higher_than_excluded}
\max_{p\in{\mathcal{P}(\muv)}}{\sum_{\uu\in{U}}{p(\uu)}} > \max_{p\in{\mathcal{P}(\muv)}}{\sum_{\uu\in{U\setminus \xx}}{p(\uu)}},
\end{align}
then it holds that:
\begin{align*}
\max_{\substack{\tilde{\muv}\in{\lclmargpoly(U)}, \tilde{\muv}\leq\muv \\ I(\xx \,;\, \tilde{\muv}) \leq 0}}~  Z(\muvt) = \max_{\substack{\tilde{\muv}\in{\lclmargpoly(U)}, \tilde{\muv}\leq\muv - \mathbf{I}_{\xx} \\ I(\xx \,;\, \tilde{\muv}) = 0}}~  Z(\muvt).
\end{align*}
\end{lemma}
\begin{lemma} \label{lem:cond_compactness_helper}
\eqref{eq:cond1} has the same optimal value as:
\begin{align} \label{eq:cond_dual_final}
\min ~ &\mu_{\xx} \\
\text{s.t. } & \muvt\in{\lclmargpoly^h}, 0 \leq \muvt \leq \tau_\mu\muv_h -  \mu_{\xx}\mathbb{I}_{\xx} \nonumber\\
& \muv_o\tau_\mu \geq 1 \nonumber \\
& \sum_{z_i}{\mut_i(z_i)} = \taut \quad \forall i\in{h} \nonumber\\
&\mu_{\xx} + \taut = 1\nonumber\\
& I(\xx_h ; \muvt) + (1-|P_h|)\taut \leq 0 \nonumber\\
& \tau_\mu I(\xx ; \muv) - \mu_{\xx} - I(\xx_h ; \muvt) + (|P_h| -1 )\taut \leq 0. \nonumber
\end{align}
\end{lemma}
The decision variables in in \eqref{eq:cond_dual_final} are $\muvt,\taut,\tau_{\mu},\mu_{\xx}$, where $\muvt$ are pseudo-marginals on hidden variables and pairs of them that are connected by an edge. This form is very similar to that of problems in \eqref{eq:maxprobs_compact}, and indeed their solutions are similar. Using \lemref{lem:cond_closed_form_helper}, we will show that a simple modification to the solution of the second problem in \eqref{eq:maxprobs_compact} leads to a solution of \eqref{eq:cond_dual_final}. This modification is shown in the following two lemmas, that also conclude the proof of \thmref{thm:struct_cond}. For now we assume the correctness of \lemref{lem:cond_compactness_helper} and \lemref{lem:cond_closed_form_helper}, their proofs are deferred to the end of this document.

To fit our problem into the formulation of \lemref{lem:cond_closed_form_helper}, define $U$ using $\bar{X}_i=\{ x_i \}$ for all observed variables $i\in{o}$ and $\bar{X}_j$ unrestricted for all hidden variables $j\in{h}$. Under this definition we have:
\begin{align*}
\max_{p\in{\linstatpoly{}}}{\sum_{\uu\in{U}}p(\uu)} &= \max_{p\in{\linstatpoly{}}}{p(\xx_o)},\\
\max_{p\in{\linstatpoly{}}}{\sum_{\uu\in{U\setminus \xx}}{p(\uu)}} &= \max_{p\in{\linstatpoly{}}}{\sum_{\zz_h\neq \xx_h}{p(\xx_o,\zz_h)}}.
\end{align*}
We are now ready to use the above lemmas and conclude the proof.
\begin{lemma} \label{lem:final1}
If $I(\xx\, ; \, \muv) \leq 0$ then \[ \min_{p\in{\mathcal{P}(\muv)}}{p(\xx_h ~ | ~ \xx_o)} = 0, \] unless $\max_{p\in{\mathcal{P}(\muv)}}{\sum_{\zz_h\neq \xx_h}{p(\zz_h, \xx_o)}} = 0$ and then the value is $1$.
\end{lemma}
\begin{proof}
We assume that $p(\xx_o)$ is constrained to be larger than $0$, otherwise the robust conditional probability problem is ill-defined. So it is trivial that if
\[
\max_{p\in{\mathcal{P}(\muv)}}{\sum_{\zz_h\neq \xx_h}{p(\xx_o,\zz_h)}} =0,
\]
then $p(\xx)=p(\xx_o)$ and the conditional is $1$. \\ Now assume towards contradiction that $\min_{p\in{\mathcal{P}(\muv)}}{p(\xx_h ~ | ~ \xx_o)} > 0$, clearly we must have:
\[
\max_{p\in{\mathcal{P}(\muv)}}{\sum_{\uu\in{U}}{p(\uu)}} > \max_{p\in{\mathcal{P}(\muv)}}{\sum_{\uu\in{U\setminus \xx}}{p(\uu)}},
\]
because otherwise equality must hold, so a maximizing distribution of the right hand side will have to achieve a conditional probability of $0$. Then the conditions of \lemref{lem:cond_closed_form_helper} hold and we have:
\[
\max_{p\in{\mathcal{P}(\muv)}}{\sum_{\zz_h\neq \xx_h}{p(\xx_o,\zz_h)}} = \max_{\substack{\tilde{\muv}\in{\lclmargpoly(U)}, \tilde{\muv}\leq\muv \\ I(\xx \,;\, \tilde{\muv}) = 0}}~  Z(\muvt).
\]

Denote the value of the above problems as $\taut_1 > 0$, let $\muvt_1$ be an optimal solution to the problem on the right hand side and $\muvt_{1,h}$ its sub-vector that corresponds to hidden variables and edges between them.
Consider taking $\muvt = \frac{\muvt_{1,h}}{\taut_1}, \taut = 1, \mu_{\xx} = 0$, we will show there exists a value of $\tau_\mu$ such that $\muvt,\taut,\mu_{\xx},\tau_{\mu}$ is a feasible solution to \eqref{eq:cond_dual_final}. The value of this solution is $\mu_{\xx}=0$, which contradicts the assumption that the minimum is strictly positive and concludes the proof. 

To see such a value of $\tau_\mu$ exists, note the following three points:
\begin{itemize}
	\item  $\muvt_1\in{\lclmargpoly(U)}, \muvt_1\leq\muv$ and normalizes to $\taut_1$. So it also holds that $\muvt \in{\lclmargpoly}, \muvt\leq \taut_1^{-1}\muv_h$, hence the first constraint of \eqref{eq:cond_dual_final} is satisfied for any $\tau_\mu \geq \taut_1^{-1}$. Also from these results it is straightforward to see that the third and fourth constraints are satisfied.
	\item Because we enforced $p(\xx_o)>0$, it holds that $\muv_o>0$. Thus the second constraint of \eqref{eq:cond_dual_final} can also be satisfied if we take a large enough value for $\tau_\mu$ (i.e. larger than one over the minimal item in $\muv_o$).
	\item Finally, we will show that
	\begin{align} \label{eq:Ieqzero}
	I(\xx_h ; \muvt) + (1-|P_h|)\taut = 0.
	\end{align}
	This means the fifth constraint is satisfied and more importantly, because $I(\xx;\muv)\leq 0$, the last constraint is satisfied for any positive value of $\tau_\mu$.\\
	To show that \eqref{eq:Ieqzero} holds, notice that:
\begin{align*} I(\xx;\muvt_1)&=0, &&\\
\mut_{1,i}(x_i)&=\taut_1 & & \forall i\in{o},\\
\mut_{1,ij}(x_i,x_j)&=\taut_1 & & \forall (i,j)\in{E_o}, \\
\mut_{1,ij}(x_i,z_j)&=\mut_{1,j}(z_j) & &  \forall (i,j)\in{E_{oh}}, z_j.
\end{align*}
	This first equality holds because it is a constraint in the problem that $\mut_1$ solves, the others because observed variables have only one possible value in $U$ and $\mut_1\in{\lclmargpoly (U)}$. Let us write down $I(\xx;\mut_1)$ and decompose the sums in its expression into smaller ones over observed and hidden variables, and to different types of edges:
	\begin{align*}
	I(\xx;\muvt_1) &= \sum_{i}{(1-d_i)\mu_i(x_i)} + \sum_{ij\in{E}}{\mu_{ij}(x_i,x_j)} \\
			    &= \sum_{i\in{o}}{(1-d_i)\taut_1} + \sum_{i\in{h}}{(1-d_i)\mut_{1,i}(x_i)} + \sum_{ij\in{E_h}}{\mut_{1,ij}(x_i,x_j)} \\
			    & + \sum_{ij\in{E_{oh}}}{\mut_{1,j}(x_j)} + \sum_{ij\in{E_o} }{\taut_1} \\
			    & = 0
	\end{align*}
	Since the subgraph of observed nodes is a forest, it has $|E_o| = |o|-|P_o|$ edges. Furthermore, $\sum_{i\in{o}}{d_i} = |E_{oh}| + 2|E_o|$ so we can rewrite the above expression as:
	\begin{align*}
	I(\xx;\muvt_1) = (|P_o| - |E_{oh}|)\taut_1 + \sum_{i\in{h}}{(1-d^h_i)\mut_{1,i}(x_i)} + \sum_{ij\in{E_h}}{\mut_{1,ij}(x_i,x_j)}.
	\end{align*}
	Notice we also combined the summation over $ij\in{E_{oh}}$ to that over $i\in{h}$, changing $d_i$ to $d_i^h$. The entire graph being a tree, it must also hold that $|E_{oh}|=|P_h|+|P_o| -1$. Plugging this into our expression, we get:
	\begin{align*}
	I(\xx;\muvt_1) = I(\xx_h;\muvt_{1,h}) + (1-|P_h|)\taut_1=0.
	\end{align*}
	Now because of the way we set $\mut$, we arrive at:
	\begin{align*}
	\frac{I(\xx;\muvt_1)}{\taut_1} = I(\xx_h;\muvt) + (1-|P_h|)\taut=0,
	\end{align*}
	which gives \eqref{eq:Ieqzero}.
\end{itemize}
Combining the items above, we see that taking $\tau_\mu$ larger than $\taut_1^{-1}$ and all entries of $\mu^{-1}_o$, gives a feasible solution as required.
\end{proof}
\begin{lemma}
If $I(\xx; \muv)>0$ then $\min_{p\in{\mathcal{P}(\muv)}}{p(\xx_h ~ | ~ \xx_o)} = \frac{I(\xx ; \muv)}{I(\xx ; \muv) + \max_{p\in{\mathcal{P}(\muv)}}{\sum_{\zz_h\neq \xx_h}{p(\zz_h, \xx_o)}}}$.
\end{lemma}
\begin{proof}
Obviously the right hand side is a lower bound on the minimum, we need to show there is a feasible solution that gives this bound. When $I(\xx; \muv) > 0$ it is easy to see that the conditions of \lemref{lem:cond_closed_form_helper} hold. So defining $\muvt_1, \taut_1$ as we did in the proof of \lemref{lem:final1}, we can assume $\muvt_1 \leq \muv - \mathbf{I}_{\xx}, I(\xx_h ; \muvt_1) + (1-|P_h|) = 0$. Now consider setting:
\begin{align*}
\tau_\mu = \frac{1}{I(\xx, \muv) + \taut_1}, ~ \muvt = \muvt_{1,h}\tau_\mu, ~ \taut = \taut_1\tau_\mu, ~ \mu_{\xx} = I(\xx; \muv)\tau_\mu.
\end{align*}
Since $\taut_1$ is defined as the value of the maximization problem in the denominator of the bound stated in the lemma, it can be seen that the value of $\mu_{\xx}$ is equal to this bound. So if this solution is feasible for \eqref{eq:cond_dual_final}, $\mu_{\xx}$ is also an upper bound on the robust conditional probability and it must also be the optimal value. We will simply go through each constraint in \eqref{eq:cond_dual_final} and show this solution satisfies it:
\begin{itemize}
	\item $\muvt\in{\lclmargpoly^h}, 0 \leq \muvt \leq \tau_\mu\muv_h -  \mu_{\xx}\mathbb{I}_{\xx_h}$: since $\muvt_1\in{\lclmargpoly (U)}$ and linear constraints stay satisfied after multiplying all variables by a positive scalar, we have $\muvt\in{\lclmargpoly^h}$. Satisfaction of capacity constraints is also a direct consequence of $\mut_1$ satisfying capacity constraints: $\muvt = \muvt_{1,h}\tau_\mu \leq (\muv_h - \mathbf{I}_{\xx})\tau_\mu = \tau_\mu\muv_h - \mu_{\xx}\mathbb{I}_{\xx_h}$.
	\item $\mu_i(x_i)\tau_\mu \geq 1 \quad \forall i\in{o}, ~ \mu_{ij}(x_i,x_j)\tau_\mu \geq 1  \quad \forall ij\in{E_o}$: Notice that $\muvt_1$ also has components for observed variables $i\in{o}$ that satisfy $\taut_1 = \mut_{1,i}(x_i)\leq \mu_i(x_i) - I(\xx ; \muv)$ and $\taut_1 = \mut_{1,ij}(x_i,x_j)\leq \mu_{ij}(x_i,x_j) - I(\xx ; \muv)$ for $ij\in{E_o}$. This gives us the constraints easily:
	\begin{align*}
	\taut_1 + I(\xx ; \muv) = \frac{1}{\tau_\mu} \leq \mu_i(x_i) \quad \forall i\in{o},
	\end{align*}
	and the same holds for every $ij\in{E_o}$.
	\item $\sum_{z_i}{\mut_i(z_i)} = \taut \quad \forall i\in{h}, \mu_{\xx} + \taut = 1$: Easy to see from our setting of $\muvt,\taut,\mu_{\xx}$, because $\muvt_1$ normalizes to $\taut_1$.
	\item $ I(\xx_h ; \muvt) + (1-|P_h|)\taut \leq 0,
 \tau_\mu I(\xx ; \muv) - \mu_{\xx} - I(\xx_h ; \muvt) + (|P_h| -1 )\taut \leq 0$: Using $I(\xx_h;\muvt) + (1-|P_h|)\taut=0$ (this was proved in the proof of \lemref{lem:final1}) and because we set $\mu_{\xx} = I(\xx; \muv)\tau_\mu$, it is easy to confirm these two constraints are satisfied.
\end{itemize}
\end{proof}
We are left with the task of proving \lemref{lem:cond_compactness_helper} and \lemref{lem:cond_closed_form_helper}, this is the topic of the next section.

\subsection{Proofs of \lemref{lem:cond_compactness_helper} and  \lemref{lem:cond_closed_form_helper}}

The problem we are concerned with, \eqref{eq:cond1}, has an exponential number of constraints. We will see shortly that these constraints can be treated as constraints on the value of 2nd-best MAP problems \cite{fromer2009lp}, one over the tree shaped field $\lambda(\zz)$ and the other over the forest shaped $\lambda(\zz_h,\xx_o)$. To prove our results we will use a relaxation of these problems. Specifically, we will use the tightness of this relaxation in trees and forests to switch these constraints with a polynomially sized set, that is easier to handle analytically. Hence we turn to derive the set of linear constraints, this is done in a very similar manner to the derivation in \cite{globerson2008fixing}.

\subsubsection{Second Best MAP using Dual Decomposition}
As proved by the authors in \cite{fromer2009lp}, the 2nd-best MAP problem over a field $\lambda(\zz)$, with excluded assignment $\xx$ can be written as follows:
\begin{align*}
\max_{\tilde{\muv}} ~ & \lambdav \cdot \tilde{\muv} \\
\text{s.t. } & \tilde{\muv}\in{\lclmargpoly}, \tilde{I}(\xx\,;\,\tilde{\muv})\leq |P|-1,
\end{align*}
where $|P|$ is the number of connected components. This is in fact a relaxation of the 2nd-best MAP problem, but it is exact when the graph is a tree or a forest. 
The dual of this problem is:
\begin{align*}
\min_{\deltav,\delta_{\xx}} ~ & \sum_i{\delta_i + \sum_{ij}\delta_{ij}}  + (|P|-1)\delta_{\xx}\\
\text{s.t. } & \lambda_{i}(z_i) + \sum_{j}{\delta_{ji}(z_i)} + (d_i-1)\delta_{\xx}\mathbb{I}_{z_i=x_i} \leq \delta_i  \quad \forall i,z_i \\
& \lambda_{ij}(z_i,z_j) - \delta_{ji}(z_i) - \delta_{ij}(z_j) -\delta_{\xx}\mathbb{I}_{z_i,z_j=x_i,x_j} \leq \delta_{ij} \quad \forall ij,(z_i,z_j) \\
& \delta_{\xx} \geq 0
\end{align*}
At the optimum, $\delta_i,\delta_{ij}$ will just be equal to the maximum of the left hand side over different values of $z_i,z_j$ (since the problem is a minimization problem), hence we can solve:
\begin{align*}
\min_{\deltav, \delta_{\xx}\geq 0}\sum_{i} & \max_{z_i}\left\{\lambda_{i}(z_i) + \sum_{j}{\delta_{ji}(z_i)} + (d_i-1)\delta_{\xx}\mathbb{I}_{z_i=x_i}\right\} + \\ & \sum_{ij}{\max_{z_i,z_j}\Bigg\{\lambda_{ij}(z_i,z_j) - \delta_{ji}(z_i) - \delta_{ij}(z_j) -\delta_{\xx}\mathbb{I}_{z_i,z_j=x_i,x_j}\Bigg\} } + (|P|-1) \delta_{\xx}
\end{align*}
To formulate a set of linear constraints that are satisfied if and only if this MAP value is smaller than a constant $c$, we can use auxiliary variables and a polynomial number of constraints, as done in \cite{DBLP:conf/icml/EbanMG14}:
\begin{align} \label{eq:2ndbestlp}
 & \sum_i{\alpha_i + \sum_{ij}\alpha_{ij}} + (|P|-1)\delta_{\xx}\leq c \\
 & \lambda_{i}(z_i) + \sum_{j}{\delta_{ji}(z_i)} + (d_i-1)\delta_{\xx}\mathbb{I}_{z_i=x_i} \leq \alpha_i  \quad \forall i,z_i \nonumber \\
& \lambda_{ij}(z_i,z_j) - \delta_{ji}(z_i) - \delta_{ij}(z_j) -\delta_{\xx}\mathbb{I}_{z_i,z_j=x_i,x_j} \leq \alpha_{ij} \quad \forall ij,(z_i,z_j) \nonumber \\
& \delta_{\xx} \geq 0.\nonumber
\end{align}
In the next section we will place these constraints in \eqref{eq:cond1} and move back to its own dual, after some manipulation this will give us \lemref{lem:cond_compactness_helper}.

\subsubsection{Concluding the Proofs}
\begin{proof}[Proof of \lemref{lem:cond_compactness_helper}]
Consider \eqref{eq:cond1}. Because we know that the optimal value of $\lambda_{\xx}$ is in the segment $[0,1]$, this problem can be written as:
\begin{align} \label{eq:cond_dual}
\max ~ & \lambda_{\xx} \\
\text{s.t. } & \max_{\zz \neq \xx}{\lambda(\zz) \leq 0} \nonumber\\
& \max_{\zz_h\neq \xx_h,\zz_o = \xx_o}{\lambda(\zz) \leq -\lambda_{\xx}} \nonumber\\
&\lambda(x_r) + \sum_{i\neq r}{\lambda_{i,pa(i)}(x_i,x_{pa(i)}) + \lambda_{i}(x_{i})} \leq 1-\lambda_{\xx} \nonumber \\
&\lambda \cdot \muv \geq 0 \nonumber.
\end{align}

%
%
%

Begin by writing the full dual problem, where we plug the liner constraints described in \eqref{eq:2ndbestlp} instead of the first two constraints in \eqref{eq:cond_dual}. The first $4$ constraints are received by replacing the first 2nd-best MAP in \eqref{eq:cond_dual}, while the $4$ constraints after these are for the second 2nd-best MAP in \eqref{eq:cond_dual}. On the right hand side we assign dual variables to each of the constraints:
\begin{align*}
&\max ~ \lambda_{\xx} \\
&\begin{array}{l | r}
\text{s.t. } \sum_i{\alpha_i + \sum_{ij}\alpha_{ij}} \leq 0 & \taub\\
 \lambda_{i}(z_i) + \sum_{j}{\deltab_{ji}(z_i)} + (d_i-1)\deltab_{\xx}\mathbb{I}_{z_i=x_i} \leq \alpha_i  \quad \forall i,z_i & \mub_i(z_i)\\
\lambda_{ij}(z_i,z_j) - \deltab_{ji}(z_i) - \deltab_{ij}(z_j) -\deltab_{\xx}\mathbb{I}_{z_i,z_j = x_i,x_j} \leq \alpha_{ij} \quad \forall ij,(z_i,z_j) & \mub_{ij}(z_i,z_j)\\
\deltab_{\xx} \geq 0 &\\
\sum_{i\in{h}}{\beta_i + \sum_{ij\in{E_h}}\beta_{ij}} + (|P_h|-1)\deltat_{\xx} \leq -\lambda_{\xx} - \sum_{ij\in{E_o}}{\lambda_{ij}(x_i,x_j)} - \sum_{i\in{o}}{\lambda_i(x_i)} & \taut \\
\lambda_{i}(z_i) + \sum_{j\in{o}}{\lambda_{ji}(x_j,z_i)} + \sum_{j\in{h}}{\deltat_{ji}(z_i)} + (d^h_i-1)\deltat_{\xx}\mathbb{I}_{z_i=x_i} \leq \beta_i  \quad \forall i\in{h},z_i & \mut_i(z_i)\\
\lambda_{ij}(z_i,z_j) - \deltat_{ji}(z_i) - \deltat_{ij}(z_j) -\deltat_{\xx}\mathbb{I}_{z_i,z_j=x_i,x_j} \leq \beta_{ij} \quad \forall ij\in{E_h},(z_i,z_j) & \mut_{ij}(z_i,z_j)\\
\deltat_{\xx} \geq 0 &\\
\lambda_r(x_r) + \sum_{i\neq r}{\lambda_{i,pa(i)}(x_i,x_{pa(i)}) + \lambda_{i}(x_{i})} \leq 1-\lambda_{\xx} & \mu_{\xx}\\
\lambda \cdot \muv \geq 0 &  \tau_\mu
\end{array}
\end{align*}
Because we assume $(V,E)$ is connected, the coefficient of $\deltab_{\xx}$ in the first constraint is $0$ and this variable does not appear in the constraint. Yet the subgraph of hidden variables might not be connected. Recall we denoted its number of connected components by $|P_h|$, this explains the coefficient of $\deltab_{\xx}$ in the fifth consraint. Now we take the dual of the above and get the problem:
\begin{align*}
&\min ~ \mu_{\xx} \\
&\begin{array}{l | r}
\text{s.t. } \mu_{\xx} + \taut = 1 & \lambda_{\xx}\\
 \mub_{i}(z_i) + \mut_{i}(z_i) - \mu_i(z_i)\tau_\mu + \mathbb{I}_{z_i=x_i}\mu_{\xx} = 0  \quad \forall i\in{h},z_i & \lambda_i(z_i), i\in{h}\\
\mub_{ij}(z_i,z_j) + \mut_{ij}(z_i,z_j) - \mu_{ij}(z_i,z_j)\tau_\mu + \mathbb{I}_{z_i,z_j=x_i,x_j}\mu_{\xx} = 0  \quad \forall ij\in{E_h},(z_i,z_j) & \lambda_{ij}(z_i,z_j) \\ 
\mub_{i}(z_i) + \mathbb{I}_{z_i=x_i}(\taut + \mu_{\xx}) - \mu_i(z_i)\tau_\mu = 0 \quad \forall i\in{o},z_i & \lambda_i(z_i), i\in{o} \\
\mub_{ij}(z_i,z_j) + \mathbb{I}_{z_i,z_j=x_i,x_j}(\taut + \mu_{\xx}) - \mu_{ij}(z_i,z_j)\tau_\mu = 0  \quad \forall ij\in{E_o},(z_i,z_j) & \lambda_{ij}(z_i,z_j) \\ 
\mub_{ij}(z_i,z_j) + \mathbb{I}_{z_j=x_j}(\mut_i(z_i) + \mathbb{I}_{z_i=x_i}\mu_{\xx}) - \mu_{ij}(z_i,z_j)\tau_\mu = 0  \quad \forall ij\in{E_{ho}},(z_i,z_j) & \lambda_{ij}(z_i,z_j)\\ 
\sum_{z_j}{\mub_{ij}(z_i,z_j)} = \mub_i(z_i) \quad \forall ij\in{E}, z_i & \deltab_{ji}(z_i)\\
\sum_{z_j}{\mut_{ij}(z_i,z_j)} = \mut_i(z_i) \quad \forall ij\in{E_h}, z_i & \deltat_{ji}(z_i)\\
\sum_{z_i}{\mub_i(z_i)} = \taub \quad \forall i  & \alpha_i \\
\sum_{z_i}{\mut_i(z_i)} = \taut \quad \forall i  & \beta_i \\
\sum_{i}{(1-d_i)\mub_i(x_i) } + \sum_{ij}{\mub_{ij}(x_i,x_j)} \leq 0  & \deltab_{\xx} \\
\sum_{i}{(1-d^h_i)\mut_i(x_i) } + \sum_{ij}{\mut_{ij}(x_i,x_j)} + (1-|P_h|)\taut \leq 0  & \deltat_{\xx} \\
\end{array}
\end{align*}
All variables in the problem are constrained to be non negative as well. The right column denotes the primal variables that each dual constraint corresponds to, in the third row these variables are $\lambda_{ij}$ for $ij\in{E_h}$, while in the fifth and sixth they are for $ij\in{E_o}$ and $E_{ho}$ respectively. Notice that we can simplify the problem by using the second to sixth equality constraints and eliminate variables $\mub$. Local consistency constraints for $\mub$:
\begin{align*}
\sum_{z_j}{\mub_{ij}(z_i,z_j)} &= \mub_i(z_i) \quad \forall ij\in{E}, z_i, \\
\end{align*}
will be satisfied because of $\mut$ and $\mu$'s local consistency, while normalization constraints:
\begin{align*}
\sum_{z_i}{\mub_{i}(z_i)} = \taub \quad \forall i,
\end{align*}
 are also satisfied because $\mut$ normalizes to $\taut$. Combining the above switch of variables into the constraint $\tilde{I}(\xx \,;\, \mub) \leq 0$, it becomes:
\begin{align*}
\tau_\mu \tilde{I}(\xx \,;\, \muv) -  \mu_{\xx} - \tilde{I}(\xx_h \,;\, \muvt) + (\sum_{i\in{o}}{(d_i-1)} - |E_o| )\taut \leq 0.
\end{align*}
We already showed in the proof of \lemref{lem:final1} that the term $\sum_{i\in{o}}{(d_i-1)} - |E_o|$ is equal to $|P_h| - 1$, turning the above constraint to:
\begin{align*}
\tau_\mu \tilde{I}(\xx \,;\, \muv) -  \mu_{\xx} - \tilde{I}(\xx_h \,;\, \muvt) + (|P_h| -1 )\taut \leq 0.
\end{align*}
So we end up with the following problem:
\begin{align*}
\min ~ &\mu_{\xx} \\
\text{s.t. } &\mu_{\xx} + \taut = 1\\
& \mut_{i}(z_i) - \mu_i(z_i)\tau_\mu + \mathbb{I}_{z_i=x_i}\mu_{\xx} \leq 0  \quad \forall i\in{h},z_i \\
& \mut_{ij}(z_i,z_j) - \mu_{ij}(z_i,z_j)\tau_\mu + \mathbb{I}_{z_i,z_j=x_i,x_j}\mu_{\xx} \leq 0  \quad \forall ij\in{E_h},(z_i,z_j)\\
& \mu_i(x_i)\tau_\mu \geq 1 \quad \forall i\in{o} \\
& \mu_{ij}(x_i,x_j)\tau_\mu \geq 1  \quad \forall ij\in{E_o} \\
& \mut_i(z_i) + \mathbb{I}_{z_i=x_i}\mu_{\xx} - \mu_{ij}(z_i,x_j)\tau_\mu \leq 0  \quad \forall ij\in{E_{ho}} \\
& \sum_{z_j}{\mut_{ij}(z_i,z_j)} = \mut_i(z_i) \quad \forall ij\in{E_h}, z_i\\
& \sum_{z_i}{\mut_i(z_i)} = \taut \quad \forall i\in{h} \\
& \tau_\mu I(\xx \,;\, \muv) -  \mu_{\xx} - I(\xx_h \,;\, \muvt) + (|P_h| -1 )\taut \leq 0 \\
& I(\xx_h ; \muvt) + (1-|P_h|)\taut \leq 0 \\
\end{align*}
Simplifying notation using the vectors $\muv_h,\mathbb{I}_{\xx}, \muv_o$ that we defined in \secref{sec:notations}, the problem takes the shape of \eqref{eq:cond_dual_final}
\end{proof}

\begin{proof}[Proof of \lemref{lem:cond_compactness_helper}]
From \thmref{thm:minmaxprobs} we know that:
\begin{align*}
\max_{\substack{\tilde{\muv}\in{\lclmargpoly(U)}, \tilde{\muv}\leq\muv}}~  Z(\muvt) &= \max_{p\in{\mathcal{P}(\muv)}}{\sum_{\uu\in{U}}{p(\uu)}}, \\
\max_{\substack{\tilde{\muv}\in{\lclmargpoly(U)}, \tilde{\muv}\leq\muv \\ I(\xx \,;\, \tilde{\muv}) \leq 0}}~  Z(\muvt) &= \max_{p\in{\mathcal{P}(\muv)}}{\sum_{\uu\in{U\setminus \xx}}{p(\uu)}}.
\end{align*}
Now for each $i,(i,j)\in{E}$, consider replacing constraints in $\mathcal{P}(\muv)$ as follows:
\begin{align*}
\sum_{\zz:z_i,z_j=x_i,x_j}{p(\zz)} &= \mu_{ij}(x_i,x_j) \rightarrow \sum_{\substack{ \zz:z_i,z_j=x_i,x_j, \\ \zz \neq \xx}}{p(\zz)} \leq \mu_{ij}(x_i,x_j) - I(\xx,\muv), \\
\sum_{\zz:z_i=x_i}{p(\zz)} &= \mu_{i}(x_i) \rightarrow \sum_{\substack{ \zz:z_i=x_i \\ \zz \neq \xx}}{p(\zz)} \leq \mu_{i}(x_i) - I(\xx,\muv).
\end{align*}
We will denote this set by $\tilde{\mathcal{P}}(\muv)$. Since for any $p\in{\linstatpoly{}}$ we know that $p(\xx) \geq I(\xx,\muv)$, it holds that $\mathcal{P}(\muv)\subseteq \tilde{\mathcal{P}}(\muv)$, which means the maximum of the new problem is \emph{higher} than that of the original for both problems (on $U$ and $U\setminus \xx$):
\begin{align*}
\max_{p\in{\mathcal{P}(\muv)}}{\sum_{\uu\in{U}}{p(\uu)}} &\leq \max_{p\in{\tilde{\mathcal{P}}(\muv)}}{\sum_{\uu\in{U}}{p(\uu)}} \\
\max_{p\in{\mathcal{P}(\muv)}}{\sum_{\uu\in{U\setminus \xx}}{p(\uu)}} &\leq \max_{p\in{\tilde{\mathcal{P}}(\muv)}}{\sum_{\uu\in{U\setminus \xx}}{p(\uu)}}
\end{align*}
Taking the dual of this new problem on $U\setminus \xx$ we obtain:
\begin{align*}
\min_{\lambdav} ~& \lambdav\cdot (\muv - \mathbf{I}_{\xx}) && \\
\text{s.t. } & \lambda(\zz) \geq 1 & & \forall \zz\in{U\setminus \xx} \\
&\lambda(\zz) \geq 0 & & \forall \zz\notin{U} \\
&\lambda_{ij}(x_i,x_j) \geq 0, \lambda_{i}(x_i) \geq 0 & & \forall i\in{V}, (i,j)\in{E}
\end{align*}
From the result in \cororef{cor:inequality_LP}, we can consider the variables to be non-negative (i.e. $\lambdav \geq 0$), the second constraint is redundant and can be removed. Furthermore, the first constraint is in fact a constraint on the value of the $2$nd-best MAP problem on $-\lambda(\zz)$ (i.e. minimization of $\lambda(\zz)$ while excluding $\xx$). Adapting the constraints in \eqref{eq:2ndbestlp} to a minimization problem and switching into our problem we get:
\begin{align} \label{eq:relaxed_dual_of_relaxed_primal}
\min_{\lambdav\geq 0, \delta_{\xx}\geq 0, \boldsymbol{\alpha},\deltav} & \lambdav\cdot (\muv - \mathbf{I}_{\xx}) \\
\text{s.t. } & \sum_i{\alpha_i + \sum_{ij}\alpha_{ij}} \geq 1 \nonumber \\
& \lambda_{i}(z_i) + \sum_{j}{\delta_{ji}(z_i)} + (1-d_i)\delta_{\xx} \mathbb{I}_{z_i=x_i} \geq \alpha_i  \quad \forall i,z_i\in{\bar{X}_i} \nonumber \\
& \lambda_{ij}(z_i,z_j) - \delta_{ji}(z_i) - \delta_{ij}(z_j) +\delta_{\xx}\mathbb{I}_{z_i,z_j=x_i,x_j} \geq \alpha_{ij} \quad \forall ij,(z_i,z_j)\in{\bar{X}_i \times \bar{X}_j}. \nonumber
\end{align}
Taking the dual of this problem, it is easy to see it equals to:
\begin{align*}
\max_{\substack{\tilde{\muv}\in{\lclmargpoly(U)}, \tilde{\muv}\leq\muv - \mathbf{I}_{\xx} \\ I(\xx \,;\, \tilde{\muv}) \leq 0}}~  Z(\muvt).
\end{align*}
The constraints of this problem are more strict than the ones in the original, therefore its value is \emph{lower}:
\begin{align*}
\max_{p\in{\mathcal{P}(\muv)}}{\sum_{\uu\in{U\setminus \xx}}{p(\uu)}} = \max_{\substack{\tilde{\muv}\in{\lclmargpoly(U)}, \muvt\leq\muv \\ I(\xx \,;\, \tilde{\muv}) \leq 0}}~  Z(\muvt) \geq \max_{\substack{\tilde{\muv}\in{\lclmargpoly(U)}, \tilde{\muv}\leq\muv - \mathbf{I}_{\xx} \\ I(\xx \,;\, \tilde{\muv}) \leq 0}}~  Z(\muvt) = \max_{p\in{\tilde{\mathcal{P}}(\muv)}}{\sum_{\uu\in{U\setminus \xx}}{p(\uu)}}.
\end{align*}
We gather that an equality must hold:
\begin{align*}
\max_{p\in{\mathcal{P}(\muv)}}{\sum_{\uu\in{U\setminus \xx}}{p(\uu)}} = \max_{p\in{\tilde{\mathcal{P}}(\muv)}}{\sum_{\uu\in{U\setminus \xx}}{p(\uu)}} = \max_{\substack{\tilde{\muv}\in{\lclmargpoly(U)}, \tilde{\muv}\leq\muv - \mathbf{I}_{\xx} \\ I(\xx \,;\, \tilde{\muv}) \leq 0}}~  Z(\muvt).
\end{align*}
To complete the proof we need to show the existence a solution $\muvt$ that is optimal for the problem on the right hand side and satisfies $I(\xx \,;\, \muvt) = 0$. Then assume towards contradiction that \eqref{eq:full_higher_than_excluded} holds and there is no optimal solution where $I(\xx \,;\, \muvt) = 0$. Since the problem is feasible, some optimal solution $\muv^*$ does exist and from complementary slackness, there is a corresponding solution $\lambdav^*,0,\boldsymbol{\alpha}^*,\deltav^*$ to \eqref{eq:relaxed_dual_of_relaxed_primal}. Since the value of $\delta_{\xx}$ is $0$, then $\lambdav^*,\boldsymbol{\alpha}^*,\deltav^*$ is also a feasible solution to the dual of:
\begin{align*}
\max_{p\in{\tilde{\mathcal{P}}(\muv)}}{\sum_{\uu\in{U}}{p(\uu)}},
\end{align*}
which means $\lambdav^* \cdot (\muv - \mathbf{I}_{\xx})$ is an upper bound on this problem. To conclude, we concatenate the inequalities we have so far:
\begin{align*}
\max_{p\in{\mathcal{P}}(\muv)}{\sum_{\uu\in{U}}{p(\uu)}} \leq \max_{p\in{\tilde{\mathcal{P}}(\muv)}}{\sum_{\uu\in{U}}{p(\uu)}} \leq \lambdav^* \cdot (\muv - \mathbf{I}_{\xx}) = \max_{p\in{\mathcal{P}(\muv)}}{\sum_{\uu\in{U\setminus \xx}}{p(\uu)}}.
\end{align*}
This inequality contradicts the hard inequality we assumed at the statement of the lemma, therefore there exists an optimal solution where $I(\xx \,;\, \muvt) = 0$ and we can incorporate this equality into the constraints without changing the value of the problem.
\end{proof}